\documentclass[letterpaper]{article}

\usepackage{graphicx} 
\usepackage{subfigure}
\usepackage{amsmath, amsthm}
\usepackage{amssymb}
\usepackage{subfigure}

\usepackage{natbib}

\usepackage{algorithm}
\usepackage{algorithmic}

\usepackage{hyperref}
 \usepackage[accepted]{icml2011}

\makeatletter

\newcommand{\Rmnum}[1]{\expandafter\@slowromancap\romannumeral #1@}
\makeatother

\icmltitlerunning{Learning from Multiple Outlooks}

\begin{document}

\theoremstyle{plain}
\newtheorem*{rep@theorem}{\rep@title}
\newcommand{\newreptheorem}[2]{%
\newenvironment{rep#1}[1]{%
 \def\rep@title{#2 \ref{##1}}%
 \begin{rep@theorem}}%
 {\end{rep@theorem}}}
\makeatother

\newtheorem{theorem}{Theorem}
\newreptheorem{theorem}{Theorem}
\newtheorem{Proof}{Proof}
\newtheorem{lemma}[theorem]{Lemma}
\newtheorem{proposition}[theorem]{Proposition}
\newtheorem{corollary}[theorem]{Corollary}
\newtheorem{definition}[theorem]{Definition}
\newtheorem{assumption}{Assumption}

\theoremstyle{remark}
\newtheorem{example}{Example}
\newtheorem{DsPnt}{DiscussionPoint}
\newtheorem{remark}{Remark}

\twocolumn[
\icmltitle{Learning from Multiple Outlooks}
\icmlauthor{Maayan Harel}{maayanga@tx.technion.ac.il}
\icmladdress{ Department of Electrical Engineering, Technion, Haifa, Israel}
\icmlauthor{Shie Mannor}{shie@ee.technion.ac.il}
\icmladdress{ Department of Electrical Engineering, Technion, Haifa, Israel}

           \vskip 0.3in
]

\begin{abstract}
We propose a novel problem formulation of learning a single task when the data are provided in different feature spaces.
Each such space is called an outlook, and is assumed to contain both labeled and unlabeled data.
The objective is to take advantage of the data from all the outlooks to better classify each of the outlooks.
We devise an algorithm that computes optimal affine mappings from different outlooks to a target outlook by matching
moments of the empirical distributions.
We further derive a probabilistic interpretation of the resulting algorithm and a sample complexity bound indicating how many samples are needed to adequately find the mapping.
We report the results of extensive experiments on activity recognition tasks that show the value of the proposed approach in boosting performance.
\end{abstract}

\section{Introduction}
It is often the case that a learning task relates to multiple representations, to which we refer as \emph{outlooks}. Samples belonging to different outlooks may have varying feature representations and distinct distributions. Furthermore, the outlooks are not related through corresponding instances, but just by the common task.

Multiple outlooks may be found in many real life problems. For example, in activity recognition when data from different users, representing the outlooks, are collected from different sensors.
Note that each outlook may have a totally different feature representations, while the recognition task is common to all outlooks. The ability to learn from these different representations is formulated by multiple outlook learning.  A different example for multiple outlooks learning is classification of document corpora written in different languages. In this case, each language represents a different outlook.
In these situations, the transformations between the outlooks are unknown and feature or sample correspondence is not available.
Consequently, it is rather difficult to learn the task at hand while exploiting the information in different representations.

The goal of multiple outlook learning is to use the information in all available outlooks to improve the learning performance of the task.
We propose to approach this learning problem in a two step procedure. First, we map the empirical distributions of the different outlooks one to another. After the outlooks' distributions are matched, a generic classification algorithm can be applied using the available examples from all the outlooks.

This approach allows to transfer an outlook of which we have little information to another where we have more information. That is, mapping the data to the same space effectively enlarges our sample size and may also give us a better representation of the problem. We show that a classifier learned in the resulting space may outperform each single classifier.

In general, matching multiple distributions, without feature alignment or assuming a parametric model, is a difficult task.
Therefore, we propose to match the empirical moments of the distributions as an approximation.
We present an algorithm for finding one such mapping. The algorithm's objective is to find the optimal affine transformations of the outlooks' spaces, while maintaining isometry within classes. From a geometric point of view, our algorithm is based on matching the centers and the main directions of the outlooks' sample distributions. One virtue of the algorithm is its simple closed form solution.


\section{Related work}
\label{sec:RelatedWork}
Learning from multiple outlooks is related to other setups such as domain adaptation, multiple view learning and manifold alignment. The main challenge in these setups, as in ours, is that the training and test data are drawn from different distributions.

Domain adaptation tries to resolve a common scenario when some changes have been made to the test distribution, while the labeling function of the domains remains more or less the same. Some authors portray this situation by assuming a single hypothesis may classify both domains well \citep{blitzer2007learning}, while others assume the target's posterior probability is equal for the domains \citep{Shimodair2000,huang2007correcting}. The latter assumption is also referred to as the covariate shift problem.

Algorithms for domain adaptation may be roughly divided to three categories. One approach is to re-weigh the training instances so they better resemble the test distribution \citep{Shimodair2000,huang2007correcting}. Such algorithms are derived from the covariate shift assumption, which is in some sense one of the outlook mapping goals. A different approach is to combine the classifiers learnt in each domain \citep{mansour2009domain}.  Last, some works suggest to change the feature representation of the domains. This may be carried out by choosing a subset of features \citep{satpal2007domain},  combination of features \citep{daume2007frustratingly}, or by finding some structural correspondence between features in different domains \citep{blitzer2006domain}. All the described approaches entail an initial common feature representation for the domains. Thus domain adaptation is a special case of the multiple outlook problem, for the case of outlooks with a common feature space. In Section \ref{experiments} we show that our approach can also be applied to this problem.

Multiple outlook learning is also closely related to the multi-view setup \citep{ruping2005learning}. In this setup, each view contains the \emph{same} set of samples represented by different features. Clearly, any multiple view data is also some instance of a multiple outlook data with the added requirement that each sample has observations from multiple outlooks. One common approach is to map a pattern matrix of each view to a consensus pattern by matching corresponding instances \citep{long2008general,hou2009multiple}. Note that in the multiple outlook framework each outlook contains a \emph{unique} set of samples, thus sample to sample correspondence is impossible. \citet{amini2009learning} considers the case when correspondence is missing for some instances, but assumes the existence of a mapping functions between the views.

Multi-view learning is sometimes referred to as manifold alignment. In manifold alignment we look for a transformation of two data sets with sample pairwise correspondence that minimizes the distance between them, in an unsupervised \citep{wang2008manifold} or a semi-supervised \citep{ham2005semisupervised} manner. \citet{wang2009manifold} present manifold alignment without pairwise correspondence. To our knowledge, this is the only work on manifold alignment that does not assume a pairwise matching of the samples. The algorithm presented in this work is not originally suited for classification as our algorithm. 

\section{Mapping Two Outlooks}
\label{sec:Multiple-outlook-Matching}

\subsection{Problem Setting}
 The learner is given two outlooks belonging to separate input spaces $\mathcal{X}_{1}$ and $\mathcal{X}_{2}$ of dimension $d^1$ and $d^2$ respectively, with a common target $\mathcal{Y}=\{1,...,c\}$. We assume that all example pairs of a given outlook $j=1,2$ are independently drawn from an unknown distribution $\mathcal{D}_{j}$, which is unique to each outlook. Denote by $X_{i}^{(1)}$ and $X_{i}^{(2)}$ the data matrices of class $i$ of outlook 1 and 2, respectively. We use superscripts to denote the outlooks' index, and subscripts to denote the classification class.

\subsection{Multiple Outlook MAPping algorithm}
\label{subsec:Matching2Outlooks}
In this section we present our main Multiple Outlook MAPping algorithm (MOMAP) for matching the representations of two outlooks. Throughout the derivations \emph{outlook 2} is mapped to \emph{outlook 1}, which is sometimes referred to as the final outlook. Our goal is to map an outlook where we have ample labeled data, to an outlook where little labeled information is available.

As a preliminary step to the mapping algorithm scaling is applied. The scaling is applied to each of the outlooks separately, and aims to normalize the features of all outlooks to the same range. Note that this stage may be done using unlabeled data when available.

Next, we use the labeled samples to match the two outlooks. The goal of this stage is to map the scaled representations by rotation and translation.  Specifically, the mapping is performed by translating the means of each class to zero, rotating the classes to fit each other well, and then translating the means of the mapped outlook to the final outlook.

Let $\left\{\hat{\mu}_i^{(1)}, \, \hat{\mu}_i^{(2)}\right\}_{i=1}^{c}$ be the set of empirical means of the outlooks. We translate the empirical means of each class of both outlooks to zero:
\begin{equation}
\label{zeroMeans}
\hat{X}_i^{(j)}= X_i^{(j)}-\hat{\mu}_i^{(j)} \quad  i=1,...,c ,\, j=1,2.
\end{equation}
Next, we turn to matching the main directions of the classes by rotation. Note that a rotation matrix may be defined in many manners. We search for mappings in the set of all orthonormal matrices (rotation and reflection). Our choice of mapping by rotation is motivated by its isometry property, which allows us to maintain the relative distance between the samples. We construct utilization matrices for each of the outlooks as follows. Define $D_{i}^{(j)}$ as the utilization matrix of outlook  $j$ and class $i$.  $D_{i}^{(1)}$ and $D_{i}^{(2)}$ are concatenated matrices constructed from the $h\leq \min(d^1,d^2)$ principal directions of the corresponding outlook and class. That is, the $h$ eigenvectors of the empirical covariance matrices $\hat{\Sigma}_{i}^{(1)},\hat{\Sigma}_{i}^{(2)}$ corresponding to the $h$ largest eigenvalues.

Using the utilization matrices we find the rotation matching the outlooks by solving the following optimization problem:
\begin{align}
\{R_{i}\} = \arg\min_{\{R_{i}\}} & \sum_{i=1}^{c}\left\Vert R_{i}D_{i}^{(2)}-D_{i}^{(1)}\right\Vert _{F}^{2}\label{eq:Rotation_opt1}\\
\textrm{subject to: }\, & R_{i}^{T}R_{i}=I \quad i=1,...,c ,\nonumber
 \end{align}
where $\left\Vert \cdot \right\Vert _{F}$ is the Frobenius norm.

To gain some intuition on Problem (\ref{eq:Rotation_opt1}) we disassemble a term in the sum of the objective function
$$\arg\min\left\Vert R_{i}D_{i}^{(2)}-D_{i}^{(1)}\right\Vert _{F}^{2}=\arg\max\sum_{l=1}^{h}\mathbf{v}_{il}^{(1)T}R\mathbf{v}_{il}^{(2)},$$
where ${\bf v}_{il}^{(j)} \, (l=1,...,h)$ are the principal directions of the $i^{th}$ class of outlook $j$. We obtain that Problem (\ref{eq:Rotation_opt1}) is equivalent to maximization of the sum of inner products between the principal directions of outlook 1 and the rotated principal directions of outlook 2,
which in turn implies minimization of the first $h$ principal angles between the classes of both outlooks.

Although Problem (\ref{eq:Rotation_opt1}) is not convex it can be solved in closed form. For the solutions constructed in this stage we borrow techniques from the literature of Procrustes Analysis \citep{gower2004procrustes}.
 Problem (\ref{eq:Rotation_opt1}) is equivalent to
 \begin{align}
& \arg\max_{R_{i}}   \sum_{i=1}^{c} tr\left(R_{i}D_{i}^{(2)}D_{i}^{(1)T}\right)\label{eq:Rotation2views_max}\\
& \textrm{subject to:} \quad  R_{i}^{T}R_{i}=I \quad i=1,...,c.\nonumber
\end{align}
Problem (\ref{eq:Rotation2views_max}) is separable, thus each component in the sum may be optimized separately. In the following derivations we drop the subscript $i$ for brevity.

Let $USV^{T}$ be the singular value decomposition (SVD) of $D^{(2)}D^{(1)T}$. Define $Z=V^{T}RU$.
Then,
\begin{align*}
&  tr\left(RD^{(2)}D^{(1)T}\right)= tr\left(RUSV^{T}\right) = \\
&  tr\left(ZS\right)=\sum_{k=1}^{d}z_{kk}\sigma_{k}\leq\sum_{i=k}^{d}\sigma_{k},
 \end{align*}
where $\sigma_{k}$ is the $k$-th singular value of $D^{(2)}D^{(1)T}$.
The upper bound is attained for $R=VU^{T}$ since in that case $Z=I$ (Algorithm \ref{alg:Proc_2outlooks}).

After the rotation, we translate the classes to match the original means of the final outlook.
The above derivation gives rise to an algorithm that matches two given outlooks.
The algorithm is described in Algorithm \ref{alg:Match2Outlooks}.

\begin{algorithm}[t]
\caption{\emph{Matching two outlooks}}
\label{alg:Match2Outlooks}
\begin{algorithmic}
\STATE \textbf{Input:} empirical moments $\hat{\mu}_{i}^{(j)} \, \forall i,j.$
\FOR{$i=1$ \textbf{to} $c$}
\STATE  $\hat{X}_i^{(j)}= X_{i}^{(j)}-\hat{\mu}_{i}^{(j)} \quad j=1,2.$
\STATE $\tilde{X}_{i}^{(2)} = MatchByRotation(\hat{X}_i^{(1)},\hat{X}_i^{(2)}).$
\STATE $X_{Mapped_i}^{(2)} = \tilde{X}_{i}^{(2)}+\hat{\mu}_{i}^{(1)}.$
\ENDFOR
\STATE \textbf{Output:} $X_{Mapped_i}^{(2)} \quad \forall i$
\end{algorithmic}
\end{algorithm}

\begin{algorithm}[t]
\caption{ \emph{MatchByRotation}}
\label{alg:Proc_2outlooks}
\begin{algorithmic}
\STATE \textbf{Input:} matrices $\hat{X}^{(1)},\hat{X}^{(2)}$.
\STATE Construct matrices $D^{(1)},D^{(2)}$.
\STATE Compute SVD factorization $D^{(2)}D^{(1)T}=USV^{T}.$
\STATE $R=VU^{T}$.
\STATE \textbf{Output:} $\tilde{X}^{(2)} = \hat{X}^{(2)}R^{T}.$
\end{algorithmic}
\end{algorithm}

\begin{remark}
Each outlook need not have the same dimension. In this case, the orthonormal constraint can not be obtained as $R$ is no longer a square matrix.
However, this problem can be easily solved. Suppose that $D_{i}^{(1)}$ and $D_{i}^{(2)}$ have different numbers of rows. Then, simply add rows of zeros to the smaller dimensional configuration until the dimensions are equalized. In this manner, we embed the smaller configuration
in the space of the larger one.
\end{remark}
\begin{remark}
Algorithm \ref{alg:Match2Outlooks} does not rely on any corresponding instances in both outlooks . However, when available, such instances may aid the mapping accuracy and can be easily incorporated into the algorithm. It is possible to do so by adding columns of the corresponding instances to the utilization matrices.
\end{remark}

\section{Extension to Multiple outlooks}
\label{MultiOutlooks}
We present an extension of Algorithm \ref{alg:Match2Outlooks} to the case of multiple outlooks. The multiple outlook scenario allows us to use the information available in all the outlooks to allow better learning of each one. 
To do so, we transform all the outlooks one to another. As for two outlooks, we begin by translating the means of each class of all the outlooks to zero. In the rotation step, the optimal rotations are found by solving
\begin{align}
& \min_{\{R_{i}^{(j)}\}}  \sum_{i=1}^{c}\sum_{k<j}\left\Vert R_{i}^{(k)}D_{i}^{(k)}-R_{i}^{(j)}D_{i}^{(j)}\right\Vert _{F}^{2} \label{eq:Rotation_opt_Multi}\\
&\textrm{subject to:} \quad R_{i}^{(j)T}R_{i}^{(j)}=I \quad \forall i,j.\nonumber
 \end{align}
Observe that Algorithm \ref{alg:Proc_2outlooks} produces an optimal solution with \emph{zero} error, as there is always a perfect rotation between two sets of $h$ orthogonal vectors.  Therefore, one optimal solution of (\ref{eq:Rotation_opt_Multi}), which attains an objective value of zero, is to rotate all outlooks to a chosen final outlook. Namely, for $m$ outlooks $m-1$ rotation matrices are computed for each class. Finally, shift the means of the rotated outlooks to those of the final outlook.

If we want to switch the choice of final outlook, all we need to do is apply the inverse mapping of the relevant outlook to all mapped outlooks. For example, to switch from outlook $s$ to $k$ one needs to apply the following transformation:
$$X_i^{(k)} = R_i^{(k)-1}\left(X_i^{(s)}-\hat{\mu}_i^{(s)}\right)+\hat{\mu}_i^{(k)} \quad \forall i.$$

\section{Analysis}
\label{sec:Analysis}
In this section we give a probabilistic robust interpretation of the rotation process, and prove a sample complexity bound on the convergence of the estimated rotation matrix .
\subsection{Probabilistic Interpretation}
In this section we discuss the effect of adding random noise to the utility matrices on the optimal rotation between two outlooks (Problem (\ref{eq:Rotation_opt1})).
We do not assume knowledge of the probability distribution of the noise. Instead, we use its bounded total value for some chosen confidence level.
We show that the solution to the noised problem is bounded by the sum of the solution to the original problem and a constant value that depends on the noise. Notably, the noise only has an additive effect to the bound.

Let $\Delta$ be the additive random uncertainty to the utility matrix $D_i^{(2)}$ for some class $i$. Suppose that this uncertainty follows an unknown joint distribution $\Delta \sim\mathcal{P}$.  This uncertainty may be portrayed by a chance-constrained extension of Problem (\ref{eq:Rotation_opt1}) \footnote{Since Problem (\ref{eq:Rotation_opt1}) is separable, the extension is done to each class separately. We drop the subscript $i$, representing the class, from the following derivations for brevity.} :
 \begin{align}
& \min_{R^{T}R=I,\tau} \tau  \label{eq:Rotation2viewsSO}\\
&    Pr_{\Delta \sim \mathcal{P}}\left\{   \left\Vert R(D^{(2)}+\Delta)-D^{(1)}\right\Vert _{F}\leq \tau \right\} \geq 1-\eta,\nonumber
\end{align}
where $\eta\in[0,1]$ is the desired confidence level.

Optimization of the chance constrained problem is natural, as it obtains, with high probability, the optimal rotation. However, despite their intuitive probabilistic form, chance constrained problems are generally intractable \citep{shapiro2009lectures}, thus we approximate Problem (\ref{eq:Rotation2viewsSO}) as follows.
We define {\small $ \rho^* = \inf_{\alpha}\left\{ Pr_{\Delta \sim \mathcal{P}}\left( \Vert \Delta \Vert_{F} \leq\alpha\right)\geq 1-\eta \right\}$} and obtain that with probability at least  $1-\eta$\\
 $$ \left\Vert R(D^{(2)}+\Delta)-D^{(1)}\right\Vert _{F} \leq \max_{\Vert \Delta \Vert_{F}\leq\rho*}\left\Vert R(D^{(2)}+\Delta)-D^{(1)}\right\Vert_{F}.$$
Therefore, Problem (\ref{eq:Rotation2viewsSO}) is upper bounded by the following minmax problem
 \begin{equation}
\min_{R^TR=I}\max_{ \Vert \Delta \Vert_{F}\leq\rho*} \left\Vert R(D^{(2)}+\Delta)-D^{(1)} \right\Vert _{F}. \label{eq:minMaxProb}
\end{equation}
This is the \emph{robust} version to the original rotation problem, with the uncertainty set $\mathcal{U} = \left\{ \Delta \, |  \, \Vert \Delta \Vert_{F}\leq\rho* \right\}$ \footnote{The original rotation problem was actually the square of the Frobenius error. However, the two problems are equivalent since taking the square does not change the solution.}. Next, we construct the robust counterpart of (\ref{eq:minMaxProb}).

\begin{theorem}{\label{robustCounterPart}}
Problem (\ref{eq:minMaxProb}) is equivalent to
$$ \min_{R^TR=I} \left( \left\Vert RD^{(2)}-D^{(1)}\right\Vert _{F} \right) +\rho^*.$$
\end{theorem}
The proof is provided in \ref{proof:robustCounterPart}. The theorem shows that Problem (\ref{eq:Rotation_opt1}) is robust to a perturbation of a total bounded value.
That is, for a bounded noise, the only difference between the solution to the original problem and its robust version (Problem (\ref{eq:minMaxProb})) is an additive constant $\rho^*$.
From a probabilistic point of view, the solution of this problem also provides a bound on the chance constrained problem in (\ref{eq:Rotation2viewsSO}).

\subsection {Sample complexity bounds}
We next provide a bound for the sample complexity of the rotation step of the algorithm.

\begin{assumption} (Gaussian Mixture)
\label{ass:GaussianMixture}
Each outlook is generated by a unique mixture of $c$ Gaussian distributions, where $c$ is the number of classes. The samples
of each outlook are realizations of $x\sim\sum_{i=1}^{c}w_{i}f_{i}(x)$,
where $\, f_{i}(x)\sim\mathcal{N}(\mu_{i},\Sigma_{i})$ and $\sum_{i=1}^{c}w_{i}=1$.
We further assume that $\Vert \mathbb{E}xx^T \Vert\leq1$ for each component.
\end{assumption}

\begin{theorem}\label{Th:SampleComplexity} 
Suppose that Assumption \ref{ass:GaussianMixture} holds. For each outlook, let $\delta,\epsilon_{i} ,\epsilon\in(0,1) ,\, (i=1,..,c)$ and suppose that the number of samples for
each class $i$ satisfies:
{\footnotesize
$$n_i  \geq  C\frac{dh^2}{\epsilon_{i}^{2}}\log^{2}\left(\frac{32dh^2}{\epsilon_{i}^{2}}\right)\log^{2}\left(\frac{4hd}{\delta}\right).$$
}
Then
$$P\left(\left\Vert \hat{R}-R\right\Vert \leq\epsilon\right) \geq  1-\delta,$$
where,  $\hat{R}$ is the estimated rotation matrix found by Algorithm \ref{alg:Proc_2outlooks},
$d$ is the dimension  and $C$ is a constant.
\end{theorem}
The proof of the theorem is provided in \ref{proof:SampleComplexity}. Note that the sample complexity of the mapping algorithm is dominated by the rotation stage. In practice, the number of chosen principal directions $h$ is usually small. Also note that the bound on the norm of the second moment in Assumption \ref{ass:GaussianMixture} is achieved by the scaling stage.

\section{Experiments}
\label{experiments}
In this section we demonstrate our framework on activity recognition data, in which different users represent different outlooks.  In this application, the multiple outlooks setup allows for valuable flexibility in real life recordings. For example, some users may use a simple sensor configuration for recordings, while others use a complex sensor board of multiple sensors. Also, this setup may resolve problems of varying sampling rates when using different hardware and workloads.

In our experiments we test two setups: a domain adaptation setup and a multiple outlook setup. For the domain adaptation setup a common feature representation is used, while for the multiple outlook setup a unique feature space is used for each user.

\subsection{Data set description and feature extraction}
\label{dataSet}
The data set used for the experiments was collected by \citet{subramanya2006recognizing} using a customized wearable sensor system. The system includes a 3-axis accelerometer, phototransistors for measuring light, barometric pressure sensors, and GPS data.
The data consist of recordings from 6 participants who were asked to perform a variety of activities and record the labels. We used the following labels: walking, running, going upstairs, going downstairs and lingering.
After removing data with obvious annotation errors the data consists of about 50 hours of recording, divided approximately evenly among the 6 users.  For each user the activities are roughly divided into $40\%$ walking, $40-50\%$ lingering, $2-5\%$ running, $2-3\%$  going upstairs, and $2-3\%$  going  downstairs.
See \citep{subramanya2006recognizing} for further details on the sensor system and the recordings.

From the raw data we extracted windowed samples as follows.
From the accelerometer data we used the x-axes measurements sampled at 512Hz, which we decimated to 32Hz. The barometric pressure sampled at 7.1Hz, was smoothed and interpolated to 32Hz. Next, we applied a two-second sliding window over each signal using a window of appropriate length. From each window a feature vector is extracted containing the Fourier coefficients of the accelerometer data, the mean of the gradient of the barometric pressure, and the mean values of the light signals. All together we obtained 20-35 thousand samples for each user with  37 features.

As explained in Section \ref{subsec:Matching2Outlooks}, before mapping the outlooks scaling should be applied to all the outlooks. For all the experiments, we scale the data to [0,1]. To reduce the sensitivity of the scaling to outliers we first collapse the extreme two percentile of the data to the value of the extreme remaining values   (also known as Winsorization). Scaling parameters are chosen on the training data and applied to the test data. This preprocessing was applied to all baseline classifiers.

\subsection{ Domain Adaptation Setup}
As mentioned above, multiple outlook learning may also be applied for domain adaptation. We tested both standard domain adaptation of two domains, as well as multiple source domain adaptation.

For the two domain problem we adopted the commonly used terminology in domain adaptation of \emph{source} and \emph{target} domains.  We applied Algorithm \ref{alg:Match2Outlooks} for different fractions of target labeled data and fully labeled source data. The performance was computed by 10-fold cross-validation, each fold containing random samples from each class according to its fraction in the complete set. The only parameter of the algorithm $h$ was chosen on a random split.

We test the success of the mapping algorithm by classification of the target test data with a classifier trained on the mapped source data, denoted as the MOMAP classifier (no target data was used for training). This is a multi-class classification problem, with five possible labels. We use a multi-class SVM classifier with an RBF-kernel ($C=64$, $\gamma=0.25$ \footnote{The parameters were chosen on the target classification problem. Common parameters were chosen for clear performance comparison of the different classifiers.}) obtained by LIBSVM software \citep{CC01a}. The data are unevenly distributed among the five classes, therefore we use the balanced error rate (BER) as a performance measure: $BER=\frac{1}{c}\sum_{i=1}^{c}\frac{1}{n_i}e_i,$ where $e_i$ and $n_i$ are the numbers of errors and number of samples in class $i$ respectively, and $c$ is the number of classes.

We compare the MOMAP classifier to the following baselines: a target only classifier, trained on the available labeled target data (TRG); a source only classifier, trained on the source data (SRC); a classifier trained on all available labeled data of target and source (ALL); and the domain adaptation algorithm presented in \citep{daume2007frustratingly} (FEDA). We also add the "optimal" error, obtained by training on the fully labeled target data (OPT).

The results are presented in Figure  \ref{fig:DA2outlooks}. It can be observed that the MOMAP classifier outperforms the baseline classifiers for most fractions of target labeled data. The algorithm performs well across all sets of users, for example, for $5\%$ labeled data it is significantly better (p-value$<0.05$) than the TRG, SRC and FEDA classifiers for all sets, and significantly better than the ALL classifier for 18 out of 30 possible sets (see Table \ref{Tb:twoUsersDA} in \ref{Res:DA}).
\begin{figure}[t]
\centering
\subfigure[User 5 $\rightarrow$ User 3]{
   \includegraphics[width=0.33\textwidth] {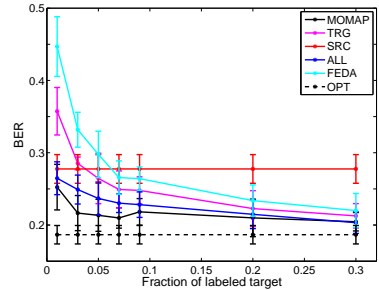}
   \label{subfig:35DA}
 }
 \subfigure[User 6 $\rightarrow$ User 2]{
   \includegraphics[width=0.33\textwidth] {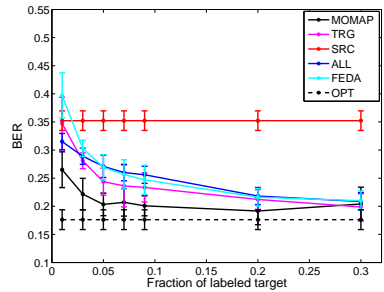}
   \label{subfig:26DA}
 }
\caption{Domain adaptation setup for 2 domains. \label{fig:DA2outlooks}}
\vspace{-3ex}
\end{figure}

In the next experiment we consider mixtures of $m$ source domains with some labeled data (both training and test sets are mixtures).
We use the extension to multiple outlooks presented in Section \ref{MultiOutlooks} to find the mappings of the sources to each outlook. We test the classification performance on each component of the mixture with a classifier trained on all the mapped sources. The final performance measure is the mean BER averaged on all the sources. As in the previous experiment, the evaluation was done by 10-fold cross-validation, with the same classifier. The baselines are similar, with the change of the TRG to the mean value of multiple classifiers trained in each domain, and the ALL baseline to a classifier trained on all sources (the SRC classifier was not relevant). The experiment was performed on all 20 triplet combinations. Sample results are presented in Figure \ref{fig:DAMultoutlooks}. These trends were consistent across users, for example, for $15\%$ of labeled data the MOMAP algorithm outperforms all other classifiers for 15 of the combinations (p-value$<0.05$). In the 5 remaining combinations, the algorithm performed significantly better than the TRG and FEDA algorithms, and equally well as the ALL classifier (see Table \ref{Tb:MultiUsersDA} in \ref{Res:DA}). For larger portions of labeled data the MOMAP algorithm also obtained smaller error than the ALL classifier (p-value$<0.05$). The effect of the ALL classifier may be a result of some regularization obtained from training on data from similar yet different domains.
\begin{figure}[t]
\centering
\includegraphics[trim = 7mm 6mm 20mm 0mm, clip,width=0.3\textwidth,height=0.25\textwidth] {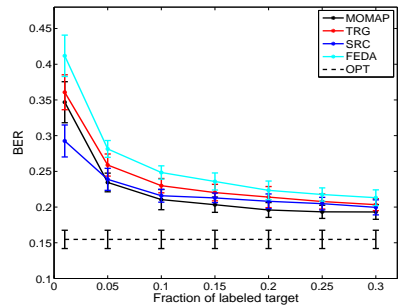}
\caption{Domain adaptation setup for multiple outlooks: users 1,2 and 5. \label{fig:DAMultoutlooks} }
\vspace{-3.5ex}
\end{figure}
\subsection{Multiple Outlook Setup}
We conducted three types of experiments for the multiple outlook setup, each with a different feature representation. The experiments' setup was similar to the previous experiments with some adjustments to the baselines: the SRC, ALL and FEDA baselines were no longer relevant, as the outlooks'  features  differ.

In the first experiment we tested the multiple outlook algorithm on two outlooks for the case of different sensors and added noise features. For the mapped outlook we used full feature representation (37 features). For the target outlook we used the accelerometer's and pressure features, and excluded the light measurements. Instead of the light features we added features with Gaussian random noise ($\mathcal{N}(0,1)$). The experiment was performed on all pair combinations. For $5\%$ labeled data of the learned outlook, the mean BER of the MOMAP was $4.5\%$ ($\pm 2.7\%$) lower than that of the TRG classifier. The results for four user pairs are presented in Figure \ref{fig:nolight}. These results show that the mapping was successful, as training on the mapped data outperforms training on partial data in the target outlook.   In Fig.~\ref{subfig:62nolight} the MOMAP algorithm has lower error than the OPT classifier for some fractions; this may be a result of the added information in the light features.
\begin{figure}[t]
\centering
\subfigure[User 3 $\rightarrow$ User 2 ]{
   \includegraphics[trim = 7mm 5mm 20mm 0mm, clip,width=0.21\textwidth,height=0.21\textwidth] {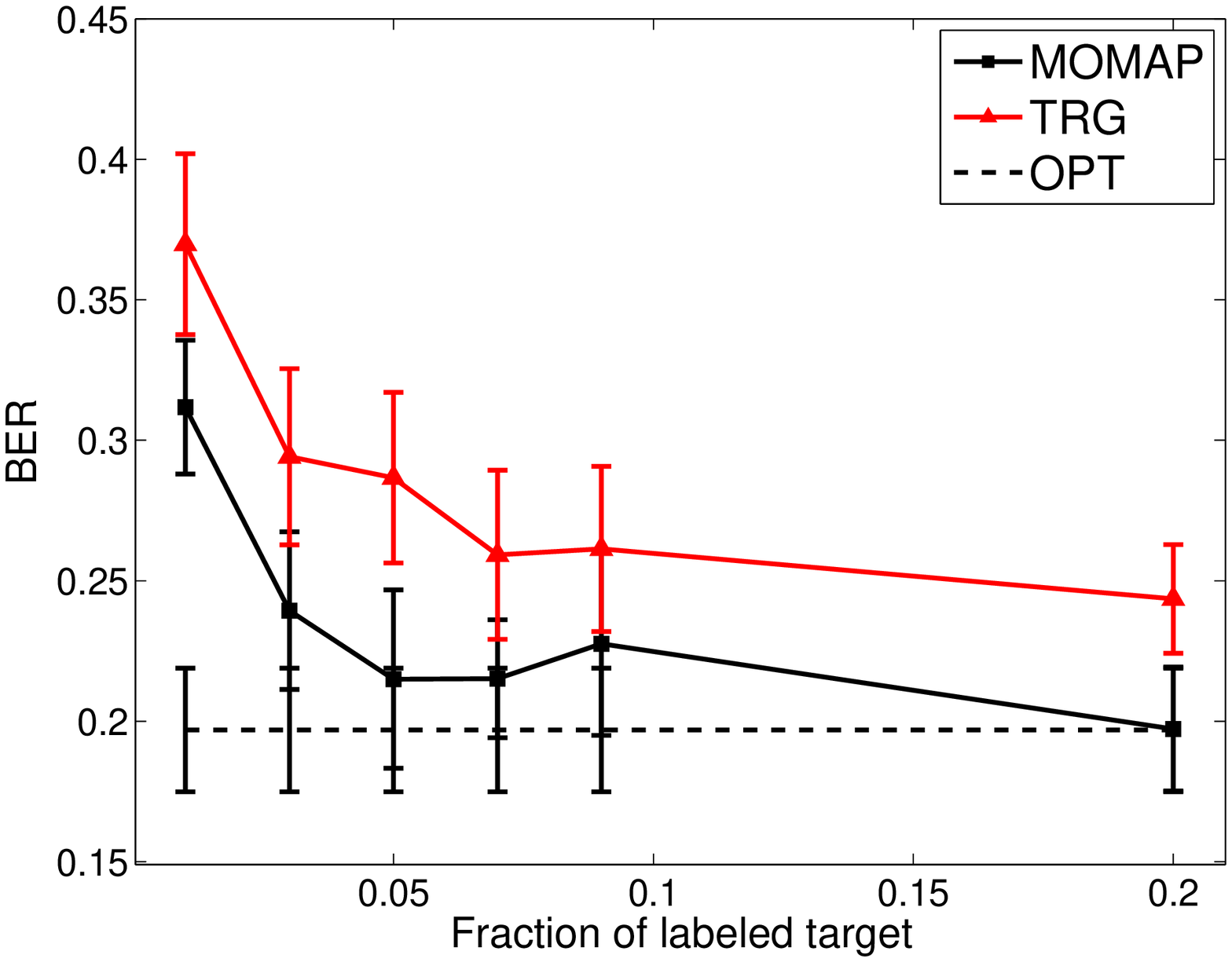} 
   \label{subfig:23nolight}
 }
\subfigure[User 4 $\rightarrow$ User 3]{
   \includegraphics[trim = 7mm 5mm 20mm 0mm, clip,width=0.21\textwidth,height=0.21\textwidth] {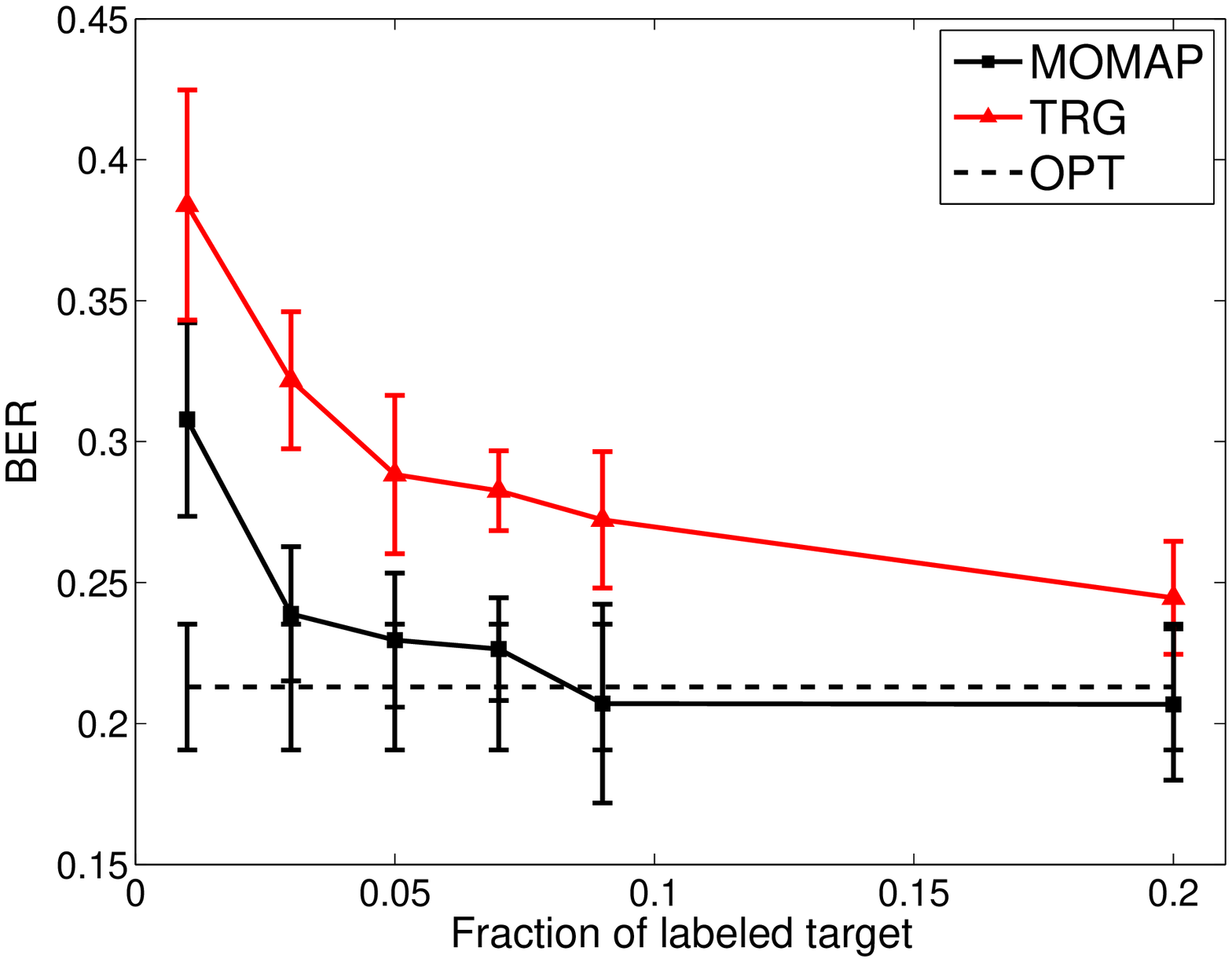}
   \label{subfig:34nolight}
 }
 \subfigure[User 2 $\rightarrow$ User 6]{
   \includegraphics[trim = 7mm 5mm 20mm 0mm, clip,width=0.21\textwidth,height=0.21\textwidth] {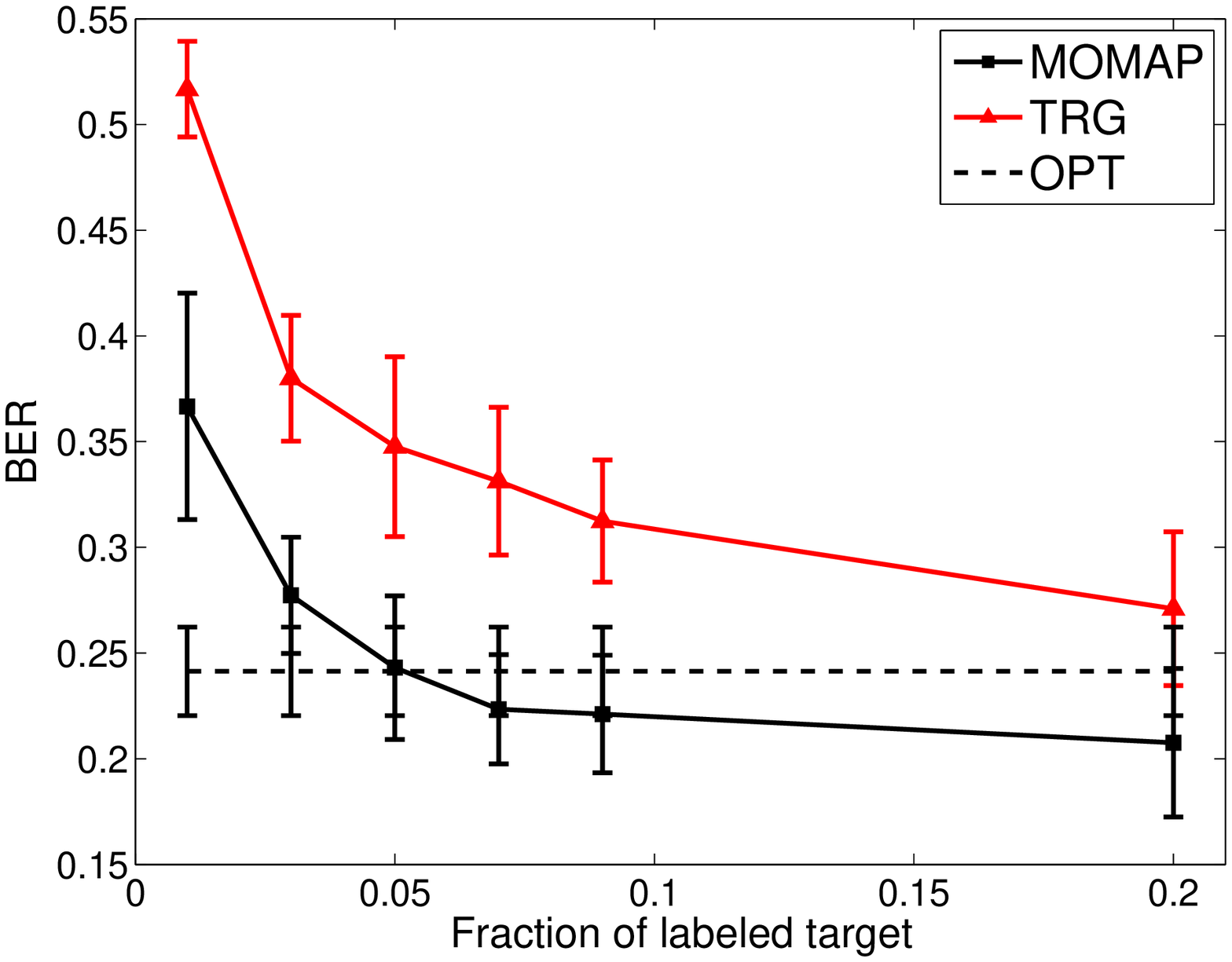} 
   \label{subfig:62nolight}
 }
  \subfigure[User 1 $\rightarrow$ User 5]{
   \includegraphics[trim = 7mm 5mm 20mm 0mm, clip,width=0.21\textwidth,height=0.21\textwidth] {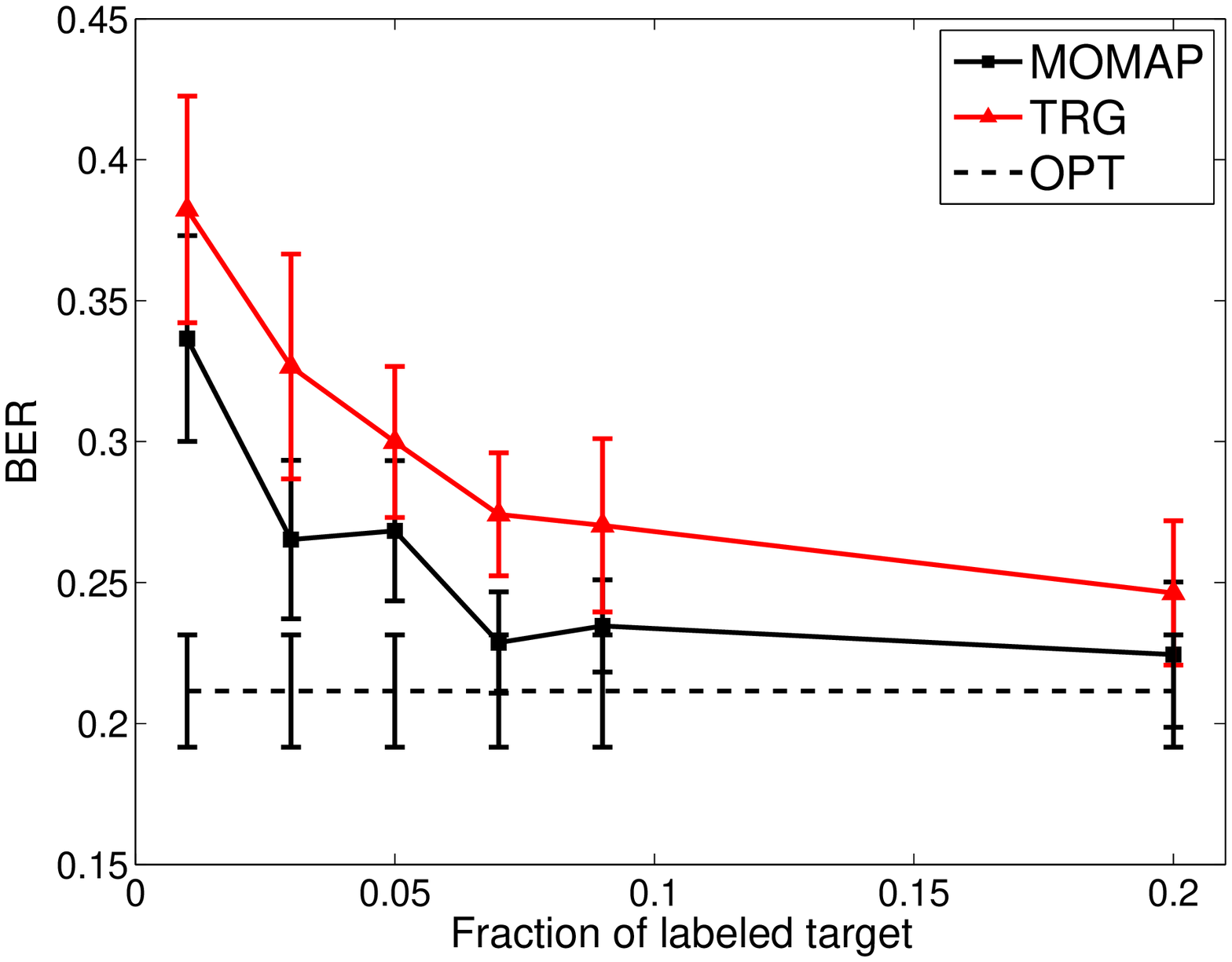} 
   \label{subfig:51nolight}
 }
 \vspace{-0.1cm}
\caption{Two outlooks with different sensors. Final outlook: accelerometer and pressure. Mapped outlook: accelerometer, pressure and light sensors. The missing features in the final outlook are replaces by noise.\label{fig:nolight} }
\vspace{-0.3cm}
\end{figure}

In the second experiment we tried to learn from two outlooks with a different number of features resulting from different sampling rates. Specifically, for the learned outlook we kept the full feature representation as described in Section \ref{dataSet}, while for the mapped outlook we used the same type of features but with 30Hz sampling rate instead of 32Hz. This resulted in 37 features in the target outlook and 35 in the mapped one. Note that our algorithm may be easily modified for this scenario; see Remark 1 in Section \ref{subsec:Matching2Outlooks}. For $5\%$ labeled data the MOMAP algorithm had on average $5.9\%$ ($\pm 2.4\%$) lower BER than the TRG classifier. Figure \ref{fig:diffFq} presents the results on four user pairs. In Figs.~\ref{subfig:13diffFq} and \ref{subfig:64diffFq} the MOMAP algorithm has lower error than the OPT classifier. Observe that this is possible since the balanced error rate is presented, which treats the error in different classes equally (namely, the MOMAP classifier does not outperform the non-balanced error).
\begin{figure}[t]
\centering
\subfigure[User 3 $\rightarrow$ User 1 ]{
   \includegraphics[trim = 7mm 5mm 20mm 0mm, clip,width=0.21\textwidth,height=0.21\textwidth] {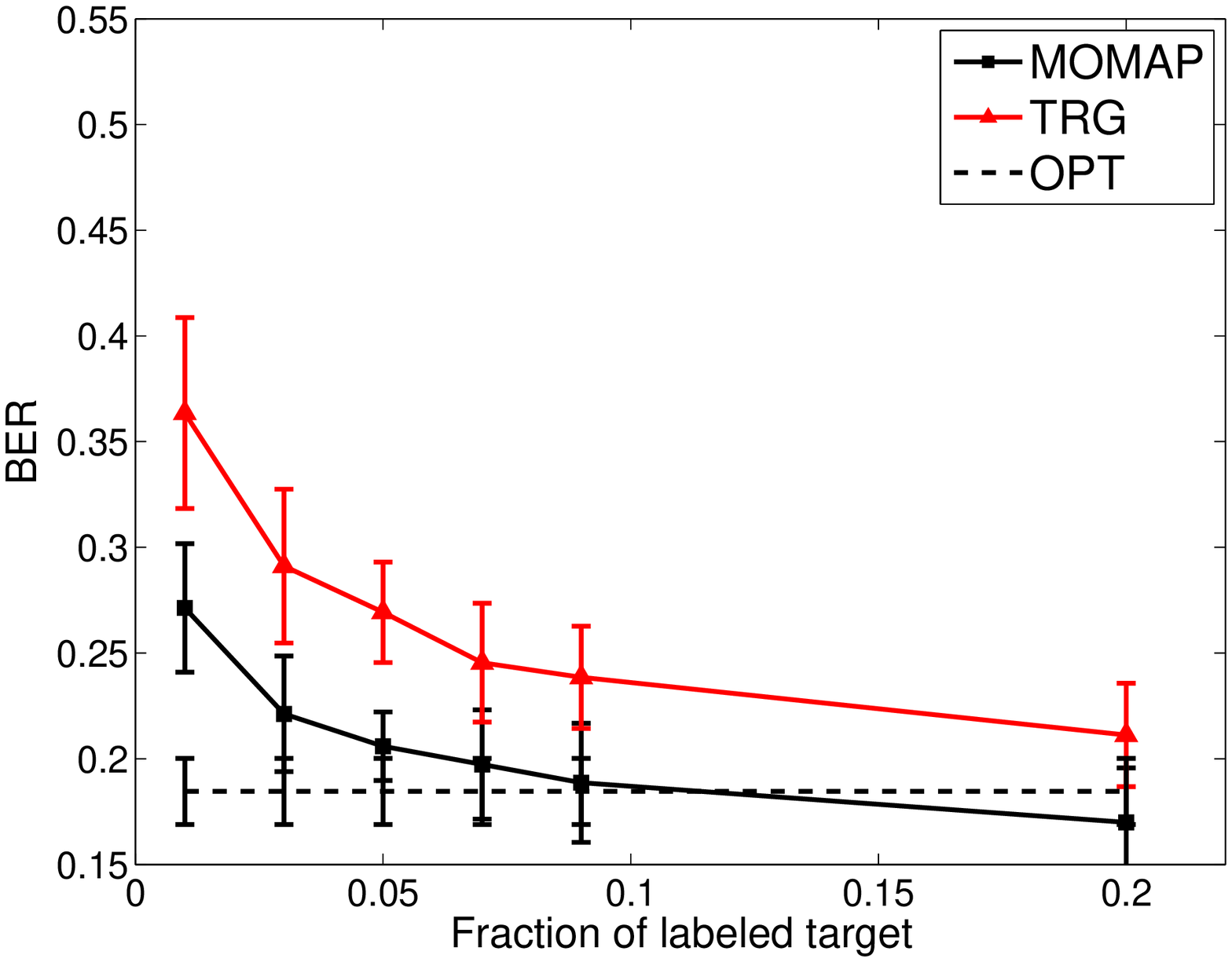}
   \label{subfig:13diffFq}
 }
\subfigure[User 2 $\rightarrow$ User 5]{
   \includegraphics[trim = 7mm 5mm 20mm 0mm, clip,width=0.21\textwidth,height=0.21\textwidth] {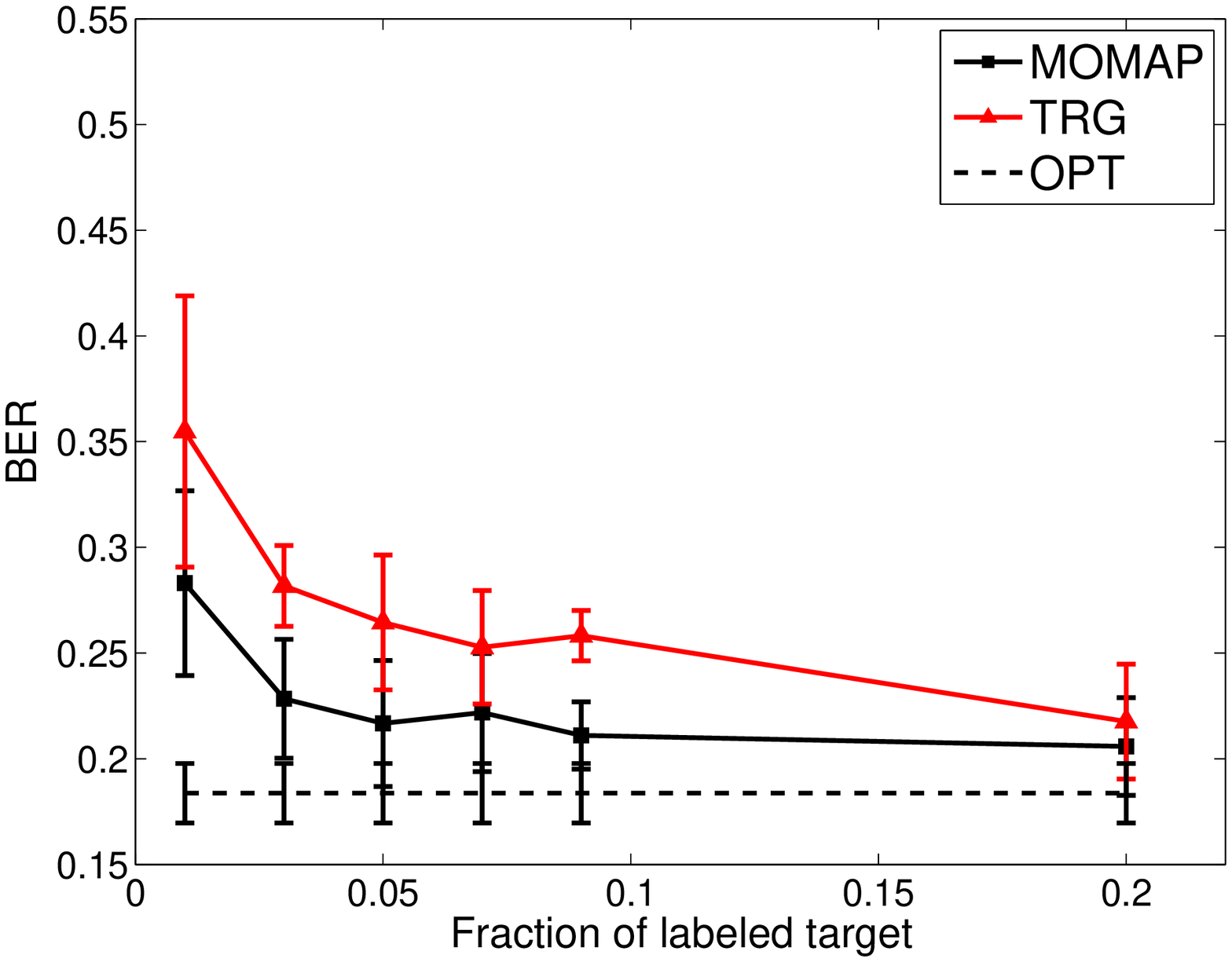}
   \label{subfig:52diffFq}
 }
 \subfigure[User 4 $\rightarrow$ User 6]{
   \includegraphics[trim = 7mm 5mm 20mm 0mm, clip,width=0.21\textwidth,height=0.21\textwidth] {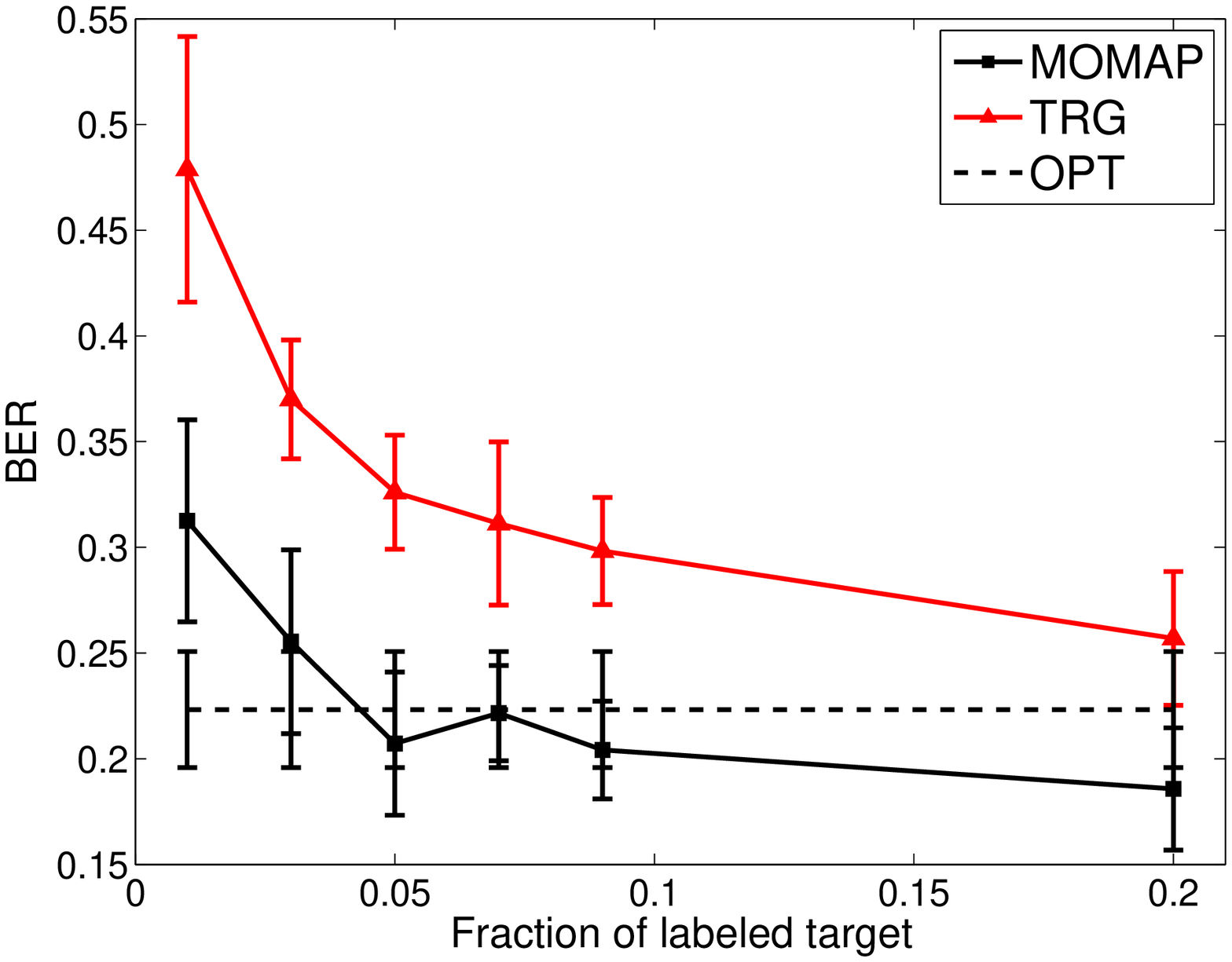}
   \label{subfig:64diffFq}
 }
  \subfigure[User 3 $\rightarrow$ User 4]{
   \includegraphics[trim = 7mm 5mm 20mm 0mm, clip,width=0.21\textwidth,height=0.21\textwidth] {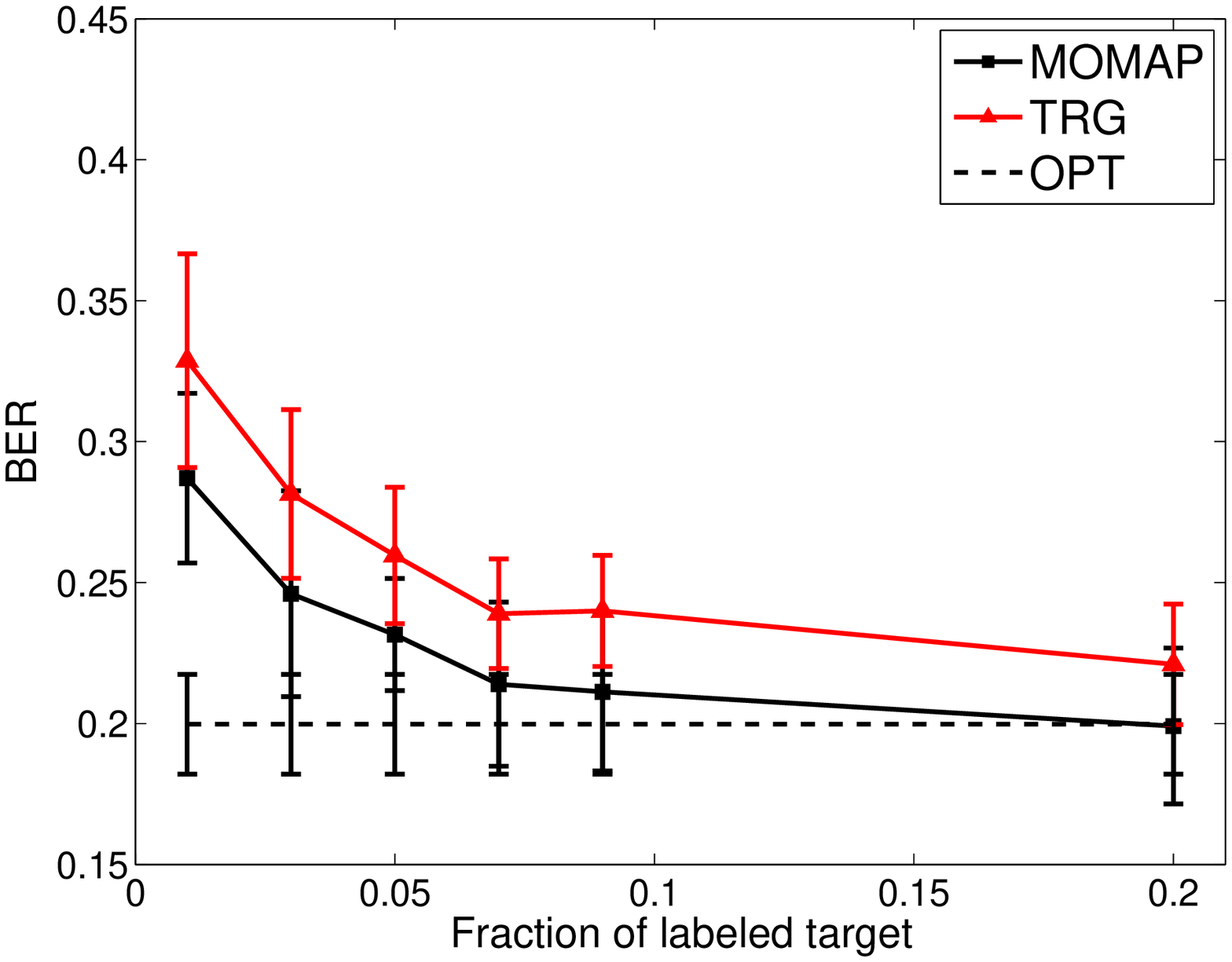}
   \label{subfig:43diffFq}
 }
\caption{Multiple outlook learning for two outlooks with different sampling rates.\label{fig:diffFq} }
\vspace{-3ex}
\end{figure}

In the third experiment we constructed the feature representation of each outlook from the 33 accelerometer's features to which we added 10 features of Gaussian noise ($\mathcal{N}(0,1)$). We then randomly permuted the order of the features of each outlook. For this experiment, we used samples belonging to the walking, running and lingering classes, as we did not use the full feature set. The experiment was performed for two outlooks as well as for multiple outlooks. The results indicate the performance boost from MOMAP especially for the running activity. Due to space limitations we provide the results in \ref{Res:MO}.

\section{Future Work}
Our proposed approach is a first step in developing the methodology for learning from multiple outlooks. This approach may be extended to many interesting directions. First, in this paper we only considered affine mappings between the outlooks and a natural extension is to consider richer classes of transformations such as piecewise linear mappings. Also, our approach is batch in the sense that first all the data have to be processed and then the classification algorithm can be used. A different extension of practical interest would be to develop an online version of the proposed approach that takes samples one by one and gradually improves the mapping.
Finally, a major application domain, of independent interest, is natural language processing. Here the challenge would be to use a language where labels are abundant to better classify in a different language. The main obstacle here seems to be the nature of representation: language data are often represented as sparse vectors which may call for a different type of transformations between the outlooks.
\vspace{-0.6cm}
\bibliographystyle{icml2011}
\bibliography{bibICML2011}

\newpage

\appendix

\section{Appendix}

\subsection{Proof of Theorem \ref{robustCounterPart}}
\label{proof:robustCounterPart}
The next theorem presents the robust counterpart of Problem (\ref{eq:minMaxProb}); the robust version of the optimization for the two outlooks rotation problem (each component in Problem \ref{eq:Rotation_opt1}). We restate the theorem for clarity:

\begin{reptheorem}{robustCounterPart}
Problem (\ref{eq:minMaxProb}) is equivalent to
$$ \min_{R^TR=I} \left( \left\Vert RD^{(2)}-D^{(1)}\right\Vert _{F} \right) +\rho^*.$$
\end{reptheorem}

\begin{proof}
We obtain an explicit expression for the maximization in (\ref{eq:minMaxProb}). By definition, the norm may be written as

{ \small \begin{align}
  & \max_{\Vert\Delta\Vert_{F}\leq\rho*} \left\Vert R(D^{(2)}+\Delta)-D^{(1)} \right\Vert _{F} =\nonumber \\
  & \max_{\Vert\Delta\Vert_{F}\leq\rho* , \Vert V \Vert_F\leq1}  tr\left( V^T(R(D^{(2)}+\Delta)-D^{(1)}) \right) =\nonumber \\
  & \max_{\Vert V \Vert_F\leq1} \left\{tr\left( V^T(RD^{(2)}-D^{(1)})\right)+\max_{\Vert\Delta\Vert_{F}\leq\rho*} tr\left( V^TR\Delta \right)\right\}.\label{RO_2outlooks}
\end{align}}

Next, we develop an explicit representation of the inner maximization over $\Delta$. By applying the Cauchy-Schwartz inequality and the unitary invariance of the Frobenius norm we obtain an upper bound:

 {\small
 $$\max_{\Vert\Delta\Vert_{F}\leq\rho*} tr\left( V^TR\Delta \right)\leq
   \max_{\Vert\Delta\Vert_{F}\leq\rho*}\Vert V\Vert_F \Vert R\Delta\Vert_F = \rho^{*}\Vert V\Vert_F.$$}\\

 Let  $\Delta^*=R^TV/\Vert V\Vert_F$.
 Observe that

{ \small $$ \max_{\Vert\Delta\Vert_{F}\leq\rho^{*}} tr\left( V^TR\Delta \right) \geq tr\left( V^TR\Delta^* \right) =\rho^{*}\Vert V \Vert_F.$$}\\

We conclude, that $\max_{\Vert\Delta\Vert_{F}\leq\rho^{*}} tr\left( V^TR\Delta \right) = \rho^{*}\Vert V \Vert_F.$ Inserting this equation into (\ref{RO_2outlooks}) we obtain:

{\small
\begin{align}
& \max_{\Vert\Delta\Vert_{F}\leq\rho*} \left\Vert R(D^{(2)}+\Delta)-D^{(1)} \right\Vert _{F} = \nonumber\\
& \max_{\Vert V \Vert_F\leq1} \left[tr\left( V^T(RD^{(2)}-D^{(1)})\right)+ \rho^{*}\Vert V \Vert_F\right] = \nonumber\\
& \left\Vert RD^{(2)}-D^{(1)}\right\Vert _{F} +\rho^{*}, \nonumber
\end{align}
}

which concludes the proof.
\end{proof}

\newpage
\subsection{ Proof of Theorem \ref{Th:SampleComplexity} }
\label{proof:SampleComplexity}
We restate the theorem for clarity:
\begin{reptheorem}{Th:SampleComplexity} (Sample complexity of rotation for two outlooks)
Suppose that Assumption \ref{ass:GaussianMixture} hold. Then, for $\delta,\epsilon_{i},\epsilon\in(0,1)$, if the number of samples for
each class and outlook $i$ satisfies:
{\footnotesize
$$n_i  \geq  C\frac{dh^2}{\epsilon_{i}^{2}}\log^{2}\left(\frac{32dh^2}{\epsilon_{i}^{2}}\right)\log^{2}\left(\frac{4hd}{\delta}\right)$$
}
then
$$P\left(\left\Vert \hat{R}-R\right\Vert \leq\epsilon\right) \geq  1-\delta,$$
where,  $\hat{R}$ is the estimated rotation matrix found by algorithm \ref{alg:Proc_2outlooks},
$d$ is the dimension  and $C$ is a constant.
\end{reptheorem}

Before providing the proof we present the following lemmas:
\begin{lemma}[Sample complexity of estimating  mean]
\label{MeanMatch}
Let Assumption \ref{ass:GaussianMixture} hold. Then for $\delta,\epsilon\in(0,1)$ if each class and outlook satisfies:
$ n \geq \frac{2d}{\epsilon^2}\log\left( \frac{d}{\delta} \right)$
then
$$P\left( \left\Vert \hat{\mu} -\mu \right\Vert \leq \epsilon\right) \geq 1-\delta,$$
where $\hat{\mu}$ and $\mu$ are the empirical and true mean of each component of the mixture.
\end{lemma}

\begin{proof}
We use $\sigma_{\max}^{2}=\max_{k}\left(\sigma_{k}^{2}\right)$ as the maximal directional variance of the $j^{th}$ mixture , and $\sigma_{k}$ as the standard deviation of the samples $k^{th}$ coordinate.
By applying Chernoff's method on each coordinate of $|\hat{\mu}_k-\mu_k|, \, k = 1,...,d$ and then applying the union bound we obtain that for $n\geq\frac{2\sigma_{\max}^{2}d}{\epsilon^{2}}\log\left(\frac{d}{\delta}\right)$
$\left\Vert \hat{\mu}-\mu\right\Vert _{2} \leq \epsilon$
holds with probability of at least $1-\delta$.
The bound is obtained by applying $ \sigma_{\max}^{2} \leq 1$, which is implied from Assumption \ref{ass:GaussianMixture}
\end{proof}

\begin{lemma}\label{lem:distFromMu}
Let X be a set of $n$ points drawn from a one dimensional Gaussian
with mean $\mu$ and variance $\sigma^{2}$. With probability $1-\delta$,
$$ \left|x-\mu\right| \leq \sigma\sqrt{2\log\left(\frac{n}{\delta}\right)} \quad \forall x\in X.$$
\end{lemma}

\begin{lemma}\label{lem:sample_norm_bound}
Let $x_{1},...,x_{n}$ be a set of independent realizations of random vectors from a multivariate normal distribution in $\mathbb{R}^{d}$.
Then with probability of at least $1-\delta$,
$$\left\Vert x_{i}\right\Vert  \leq \left\Vert \mu\right\Vert +\sigma\sqrt{2d\log\left(\frac{nd}{\delta}\right)}.$$
\end{lemma}

\begin{proof}
By the reverse triangle inequality we have that
$$ \left\Vert x_{i}\right\Vert -\left\Vert \mu\right\Vert \leq\left|\left\Vert x_{i}\right\Vert -\left\Vert \mu\right\Vert \right|\leq\left\Vert x_{i}-\mu\right\Vert .$$
By applying Lemma \ref{lem:distFromMu} on a single coordinate of the random vectors $x_{i}$ we get
{\small
$$P\left(\left|x_{i}^{(k)}-\mu_{k}\right|\geq\frac{\epsilon}{\sqrt{d}}\right) \leq n\textrm{exp}\left(-\frac{1}{2}\frac{\epsilon^{2}}{\sigma^{2}d}\right)\leq\frac{\delta}{d}.$$
}
Taking the union bound over the $d$ coordinates we get that with probability at least $1-\delta$
\begin{eqnarray*}
\left\Vert x_{i}\right\Vert -\left\Vert \mu\right\Vert  & \leq & \left\Vert x_{i}-\mu\right\Vert \leq\sigma\sqrt{2d\log\left(\frac{nd}{\delta}\right)}.\end{eqnarray*}
\end{proof}

\begin{lemma}[Sample complexity of covariance estimation]
\label{Cov_est}
Let $X$ be a set of random samples generated from a Gaussian distribution with covariance $\Sigma$ and zero mean $\mu=0.$
Define $\hat{\Sigma},\,\hat{\mu}$ as the estimated covariance matrix and mean of the sample.
Then for $\delta,\,\epsilon_{1},\,\epsilon_{2}\in(0,1)$, for a sample size of
$$n  \geq  C\frac{d}{\epsilon_{2}^{2}}\log^{2}\left(\frac{2d}{\epsilon_{2}^{2}}\right)\log^{2}\left(\frac{2d}{\delta}\right)$$
we have that
$$P\left(\left\Vert \hat{\Sigma}-\Sigma \right\Vert \leq\epsilon_{1}+\epsilon_{2}\right) \geq  1-\delta.$$
\end{lemma}

\begin{proof}
The concentration bound is obtained by dividing the error to two components,
{\footnotesize
\begin{eqnarray}
\left\Vert \hat{\Sigma}-\Sigma\right\Vert  \leq  \left\Vert \mu\mu^{T}-\hat{\mu}\hat{\mu}^{T}\right\Vert +
\left\Vert \frac{1}{n}\sum_{i=1}^{n}x_{i}x_{i}^{T}-\mathbb{E}xx^{T}\right\Vert,   \label{eq:cov_inq}
\end{eqnarray}
 }{\footnotesize \par}

We begin by bounding the first component: \\
Recall that $\mu=0$, so the first component is bounded by $\Vert \hat\mu \Vert^2$. We apply Lemma \ref{MeanMatch} and obtain that with probability at least $1-\frac{\delta}{2}$:
\begin{align}
& n_1 \geq \frac{2d}{\epsilon} \log \left( \frac{2d}{\delta} \right), \label{muComp}\\
& \Vert \hat\mu \Vert^2 \leq \epsilon_1. \nonumber
\end{align}

The second component is bounded by a concentration inequality for covariance matrices presented by \citet{rudelson2007sampling}. For completeness we add the relevant theorem; see Theorem \ref{Rudelson}.
The second moment condition holds by Assumption \ref{ass:GaussianMixture}. The second condition, of bounded sample norm  is obtained as follows.
By applying Lemma \ref{lem:sample_norm_bound} and bounding the variance according to  Assumption \ref{ass:GaussianMixture}, we get that
$\left\Vert x_{i}\right\Vert  \leq\sqrt{2d\log\left(\frac{n_2d}{\delta}\right)}.$

Next, we apply Theorem 3.1 of \cite{rudelson2007sampling} with $t^{2}=a^{2}\log\left(\frac{2}{\delta}\right)/c$
and $a=\epsilon_2\sqrt{c/\log\left(\frac{2}{\delta}\right)}$. This
results in the condition
$$a=\frac{\epsilon_2 c}{\sqrt{\log(\frac{2}{\delta})}}  \geq C\frac{\sqrt{2d\log\left(\frac{nd}{\delta}\right)\log(n)}}{\sqrt{n}},$$
 which is satisfied for the choice of\\
\begin{eqnarray}
n_2 & \geq & C\frac{d}{\epsilon_2^{2}}\log^{2}\left(\frac{2d}{\epsilon_2^{2}}\right)\log^{2}\left(\frac{2d}{\delta}\right).
\label{eq:Th1_cov_bound}
\end{eqnarray}
 We get the final sample bound by taking the maximum between the sample
complexity of the mean (\ref{muComp}) and the covariance estimation (\ref{eq:Th1_cov_bound}).
\end{proof}

\begin{proof}[Proof of Theorem \ref{Th:SampleComplexity}]
Observe that by applying Equation (\ref{zeroMeans}) to each class and outlook we have that each component has zero mean. By Lemma \ref{MeanMatch}, the sample complexity of this step is $n_i\geq \frac{2d}{\epsilon^2}\log\left( \frac{d}{\delta} \right)$ (for each class and outlook $i$). In the following derivations we assume zero mean of the components' distribution. We show that the sample complexity of both stages is dominated by the rotation.

By substituting the finite and infinite sample rotation matrices with
the values defined in Alg.~\ref{alg:Proc_2outlooks} and applying
the triangular inequality twice we have that
 \begin{align}
 & \left\Vert \hat{R}-R\right\Vert _{F}=\left\Vert \hat{V}\hat{U}^{T}-VU^{T}\right\Vert _{F}\nonumber \\
 & \leq\left\Vert V\right\Vert \left\Vert \Delta U\right\Vert +\left\Vert \Delta V\right\Vert \left\Vert \Delta U\right\Vert +\left\Vert \Delta V\right\Vert \left\Vert U\right\Vert \,,\label{eq:rot_proof_1}
 \end{align}

where $\Delta V=\hat{V}-V$ and $\Delta U=\hat{U}-U$. Recall that the matrices $U,\,\hat{U},\, V,\,\hat{V}$ are the matrices of singular
vectors resulting from the SVD decompositions $\hat{D}^{(2)}\hat{D}^{(1)T}=\hat{U}\hat{S}\hat{V}^{T}$
and $D^{(2)}D^{(1)T}=USV^{T}$. We apply the perturbation theory of the
SVD decomposition presented in \cite{Stewart1990Perturb}
and bound Eq.~(\ref{eq:rot_proof_1}) by
\begin{eqnarray}
\left\Vert \hat{R}-R\right\Vert _{F} & \leq C\left\Vert D^{(2)}D^{(1)T}-\hat{D}^{(2)}\hat{D}^{(1)T} \right\Vert _{F} & ,\label{eq:rot_proof_2}
\end{eqnarray}
where $C$ is a constant.
Observe that
\begin{align}
        &\left\Vert  D^{(2)}D^{(1)T}-\hat{D}^{(2)}\hat{D}^{(1)T}  \right\Vert _{F} \nonumber \\
  \leq  &\left\Vert D^{(2)}D^{(1)T}  - D^{(2)}\hat{D}^{(1)T}  \right\Vert _{F} + \left\Vert  D^{(2)}\hat{D}^{(1)T} - \hat{D}^{(2)}\hat{D}^{(1)T} \right\Vert _{F}\nonumber\\
  \leq  &\sqrt{h}\left\Vert  D^{(1)T}-\hat{D}^{(1)T} \right\Vert _{F}   + \sqrt{h}\left\Vert  D^{(2)} - \hat{D}^{(2)} \right\Vert _{F} \nonumber \\
  \doteq & \sqrt{h} \left( \Vert \Delta D_1 \Vert_F + \Vert \Delta D_2 \Vert_F   \right) \label{deltaDs}.
\end{align}
The second inequality holds by the sub-multiplicative property of the Frobenius norm. When the number of columns $h<d$, columns of zeros need to be added to make the matrices square.

Define $ \mathbf{v}_{l}^{(i)}$ and $\hat{\mathbf{v}}_{l}^{(i)}$ ($\, l=1,...,h$) to be the $h$ eigenvectors of matrices $D^{(i)}$ and $\hat{D}^{(i)}$ respectively.
The following holds
$\left\Vert \Delta D_{i}\right\Vert _{F}^{2}=\sum_{l=1}^{h}\left\Vert \hat{\mathbf{v}}_{l}^{(i)}-\mathbf{v}_{l}^{(i)}\right\Vert_2 ^{2}$
by definition.
Define the perturbation of the covariance matrix of mixture $i$ by $E_{i}=\Sigma_{i}-\hat{\Sigma}_{i}$.
By applying the perturbation theory of the eigen decomposition on the perturbed
covariance matrices \cite{Stewart1990Perturb} (p.240) we get that
$\left\Vert\hat{\mathbf{v}}_{l}^{(i)}-\mathbf{v}_{l}^{(i)} \right\Vert  \leq C\left\Vert E_{i}\right\Vert .$

Last, we use Lemma \ref{Cov_est} to bound $E_{i}$ for each outlook ($i=1,2$).
If the number of samples for each outlook is
$$n_i  \geq  C\frac{dh^2}{\epsilon_{i2}^{2}}\log^{2}\left(\frac{32dh^2}{\epsilon_{i2}^{2}}\right)\log^{2}\left(\frac{4hd}{\delta}\right)$$
then
$$P\left( \left\Vert \hat{\Sigma}_{i}-\Sigma_{i} \right\Vert \leq \frac{\epsilon_{i,1}+\epsilon_{i,2}}{4h} \right) \geq 1-\frac{\delta}{2h}, $$
which implies
$$P\left(\left\Vert \Delta D_{i}\right\Vert _{F} \leq \frac{\epsilon_{i,1}+\epsilon_{i,2}}{4\sqrt{h}}\right)   \geq  1-\frac{\delta}{2}. $$
Plugging in the bound to (\ref{deltaDs}) we get the final bound:
$$ P\left(\left\Vert D^{(2)}D^{(1)T}-\hat{D}^{(2)}\hat{D}^{(1)T} \right\Vert _{F} \leq  \epsilon \right)   \geq  1-\delta,$$
for some $\epsilon = \frac{1}{4} \sum_{i=1,2} \epsilon_{i,1}+\epsilon_{i,2} \in (0,1).$

\end{proof}

\begin{theorem}[Theorem 3.1 from \citep{rudelson2007sampling}]
\label{Rudelson}
Let $x$ be a random vector in $\mathbb{R}^{d}$ from distribution $D$, which is uniformly bounded almost everywhere: $\left\Vert x\right\Vert \le M$,
and $\left\Vert \mathbb{E}xx^{T}\right\Vert \le1$. Let $x_{1}...x_{n}$
be independent samples generated from $D$. Define
$$a  =  CM\sqrt{\frac{\log n}{n}},$$
where $C$ is an absolute constant. Then, for every $t\in(0,1)$,
$$P\left(\left\Vert \frac{1}{n}\sum_{i=1}^{n}x_{i}x_{i}^{T}-\mathbb{E}\left(xx^{T}\right)\right\Vert >t\right)  \leq  2e^{-ct^{2}/a^{2}}.$$
\end{theorem}
\newpage
\subsection{Domain Adaptation setup - Results}
\label{Res:DA}
Following are results obtained on all users for the domain adaptation experiment. Table  \ref{Tb:twoUsersDA} presents the results for two users obtained on $5\%$ labeled target data. Table \ref{Tb:MultiUsersDA} presents the results for multi-source domain adaptation with three users, each with $15\%$ labeled data. Both tables contain the balanced error rate (BER) on the five class classification task. Highlighted results represent significance of the result with p-value$<0.05$.

\begin{table}[h!]
\caption{Domain Adaptation setup for two users \\($5\%$ labeled Target).}
\label{Tb:twoUsersDA}
{ \small
\begin{tabular}{|c|l|l|l|l|l|}
\hline
                        \multicolumn{6}{c}{}
                        \\\hline
                    S $\rightarrow$ T  &   MOMAP   & FEDA   & TRG &  SRC  & ALL  \\
                    \hline
                     2 $\rightarrow$  1 & {\bf 0.208} & 0.280 &0.249  &0.255  &0.234\\
                     3 $\rightarrow$  1 & 0.228       & 0.292 &0.269  &0.209  &{\bf 0.2}\\
                     4 $\rightarrow$  1 & {\bf 0.221} & 0.293 & 0.256 & 0.246 & 0.233 \\
                     5 $\rightarrow$  1 &       0.21  & 0.304 & 0.27  & 0.23  & 0.216 \\
                     6 $\rightarrow$  1 &      0.255  & 0.294 & 0.265 & 0.345 & 0.283 \\
                     1 $\rightarrow$  2 &      0.20   & 0.29  & 0.26  & 0.21  & 0.20 \\
                     3 $\rightarrow$  2 & 0.212       & 0.281 & 0.253 & 0.215 & 0.205\\
                     4 $\rightarrow$  2 & {\bf 0.186} & 0.287 & 0.252 & 0.216 & 0.209 \\
                     5 $\rightarrow$  2 & {\bf 0.191} & 0.281 & 0.249 & 0.223 & 0.208\\
                     6 $\rightarrow$  2 & {\bf 0.203} & 0.27  & 0.244 & 0.352 & 0.271\\
                     1 $\rightarrow$  3 & {\bf 0.216} & 0.281 & 0.26  & 0.23  & 0.224\\
                     2 $\rightarrow$  3 & {\bf 0.214} & 0.271 & 0.256 & 0.265 & 0.241\\
                     4 $\rightarrow$  3 &       0.215 & 0.276 & 0.252 & 0.233 & 0.222\\
                     5 $\rightarrow$  3 & {\bf 0.213} & 0.298 & 0.264 & 0.278 & 0.237\\
                     6 $\rightarrow$  3 & {\bf 0.210} & 0.276 & 0.251 & 0.359 & 0.282\\
                     1 $\rightarrow$  4 & {\bf 0.233} & 0.277 & 0.256 & 0.309 & 0.253\\
                     2 $\rightarrow$  4 & {\bf 0.231} & 0.269 & 0.264 & 0.314 & 0.265\\
                     3 $\rightarrow$  4 & 0.245 & 0.281 & 0.27  & 0.276 & 0.249\\
                     5 $\rightarrow$  4 & 0.235 & 0.289 & 0.27  & 0.313 & 0.246\\
                     6 $\rightarrow$  4 & {\bf 0.243} & 0.267 & 0.262 & 0.422 & 0.293\\
                     1 $\rightarrow$  5 &   0.228 & 0.307 & 0.272 & 0.244 & 0.237\\
                     2 $\rightarrow$  5 & {\bf 0.237} & 0.29  & 0.275 & 0.289 & 0.267\\
                     3 $\rightarrow$  5 &       0.233 & 0.289 & 0.261 & 0.239 & 0.228\\
                     4 $\rightarrow$  5 & {\bf 0.22}  & 0.286 & 0.258 & 0.258 & 0.243\\
                     6 $\rightarrow$  5 & {\bf 0.221} & 0.269 & 0.247 & 0.3   & 0.259\\
                     1 $\rightarrow$  6 & {\bf 0.234} & 0.376 & 0.321 & 0.294 & 0.273\\
                     2 $\rightarrow$  6 & {\bf 0.238} & 0.37  & 0.316 & 0.305 & 0.273\\
                     3 $\rightarrow$  6 & 0.254       & 0.386 & 0.344 & 0.261 & 0.247 \\
                     4 $\rightarrow$  6 & {\bf 0.235} & 0.374 & 0.326 & 0.294 & 0.263 \\
                     5 $\rightarrow$  6 & 0.244 & 0.379 & 0.325 & 0.246 & 0.239\\

\hline
\end{tabular}
}
\end{table}

\begin{table}
\caption{Domain Adaptation setup - Multi-users \\($15\%$ labeled Target). }
\label{Tb:MultiUsersDA}
{ \small
\begin{tabular}{|c|l|l|l|l|}
\hline
                        \multicolumn{5}{c}{} \\
\hline

                    Users   &   MOMAP   & FEDA   & TRG   & ALL  \\
\hline
                    1 2 3   & {\bf 0.205} & 0.232 & 0.227 & 0.214 \\
                    1 2 4   & {\bf 0.203} & 0.235 & 0.224 & 0.214\\
                    1 2 5   & {\bf 0.203} & 0.236 & 0.22  & 0.213\\
                    1 2 6   & {\bf 0.211} & 0.253 & 0.238 & 0.226\\
                    1 3 4   & {\bf 0.207} & 0.233 & 0.224 & 0.22\\
                    1 3 5   &  0.208      & 0.24  & 0.226 & 0.21 \\
                    1 3 6   & {\bf 0.221} & 0.255 & 0.239 & 0.228\\
                    1 4 5   & {\bf 0.208} & 0.237 & 0.223 & 0.219\\
                    1 4 6   & {\bf 0.214} & 0.252 & 0.236 & 0.232\\
                    1 5 6   & {\bf 0.222} & 0.257 & 0.239 & 0.228\\
                    2 3 4   & 0.214       & 0.234 & 0.229 & 0.216 \\
                    2 3 5   & {\bf 0.21}  & 0.235 & 0.228 & 0.215\\
                    2 3 6   & {\bf 0.218} & 0.243 & 0.236 & 0.225\\
                    2 4 5   & {\bf 0.204} & 0.233 & 0.221 & 0.212\\
                    2 4 6   & {\bf 0.216} & 0.254 & 0.239 & 0.232\\
                    2 5 6   & 0.226       & 0.257 & 0.243 & 0.226\\
                    3 4 5   & 0.219       & 0.239 & 0.231 & 0.222 \\
                    3 4 6   & {\bf 0.224} & 0.258 & 0.244 & 0.235\\
                    3 5 6   &  0.227      & 0.254 & 0.239 & 0.225 \\
                    4 5 6   & {\bf 0.222} & 0.252 & 0.242 & 0.232\\

\hline
\end{tabular}
}
\end{table}

\newpage
\subsection{Multiple outlook setup - Experiment 3}
\label{Res:MO}
In the third experiment we constructed the feature representation of each outlook from the 33 accelerometer's Fourier coefficients to which we added 10 features of random Gaussian noise $\mathcal{N}(0,1)$. We then randomly permuted the order of the features of each outlook. For this experiment, we used samples belonging to the walking, running and lingering classes, as we did not use the full feature set. The experiment was performed for the two outlook scenario as well as for multiple outlooks.

Figure \ref{fig:diffFq} shows the results for $5\%$ labeled target data for different users couples. It can be observed, that for the walking and lingering activities the mapped outlook performs similarly to the TRG classifier. For all cases, the mapped outlook classifies the running activity with least errors. Among all user pairs the MOMAP classifier obtained smaller error for the running activity ($3.5\%-45\%$ smaller for $5\%$ labeled data). The results show the boosting power of the mapping, which, as may be expected, is most powerful for the classes with less labeled data. An interesting behavior is that even when all labeled data is available the MOMAP algorithm sometimes outperforms the classifier learned in the target outlook (OPT). This may be caused by some regularization obtained by the mapping. Note, however, that for the total error, on all three classes, the MOMAP classifier does not outperform OPT classifier.
The results for multiple outlooks are presented in Figure \ref{fig:MultnoisePerm}. It can be observed that the mapping aids in learning the mixture.

\begin{figure}[h!]
\centering
\subfigure[User 2 $\rightarrow$ User 1 ]{
   \includegraphics[trim = 12mm 0mm 0mm 0mm, clip,width=0.22\textwidth,height=0.21\textwidth] {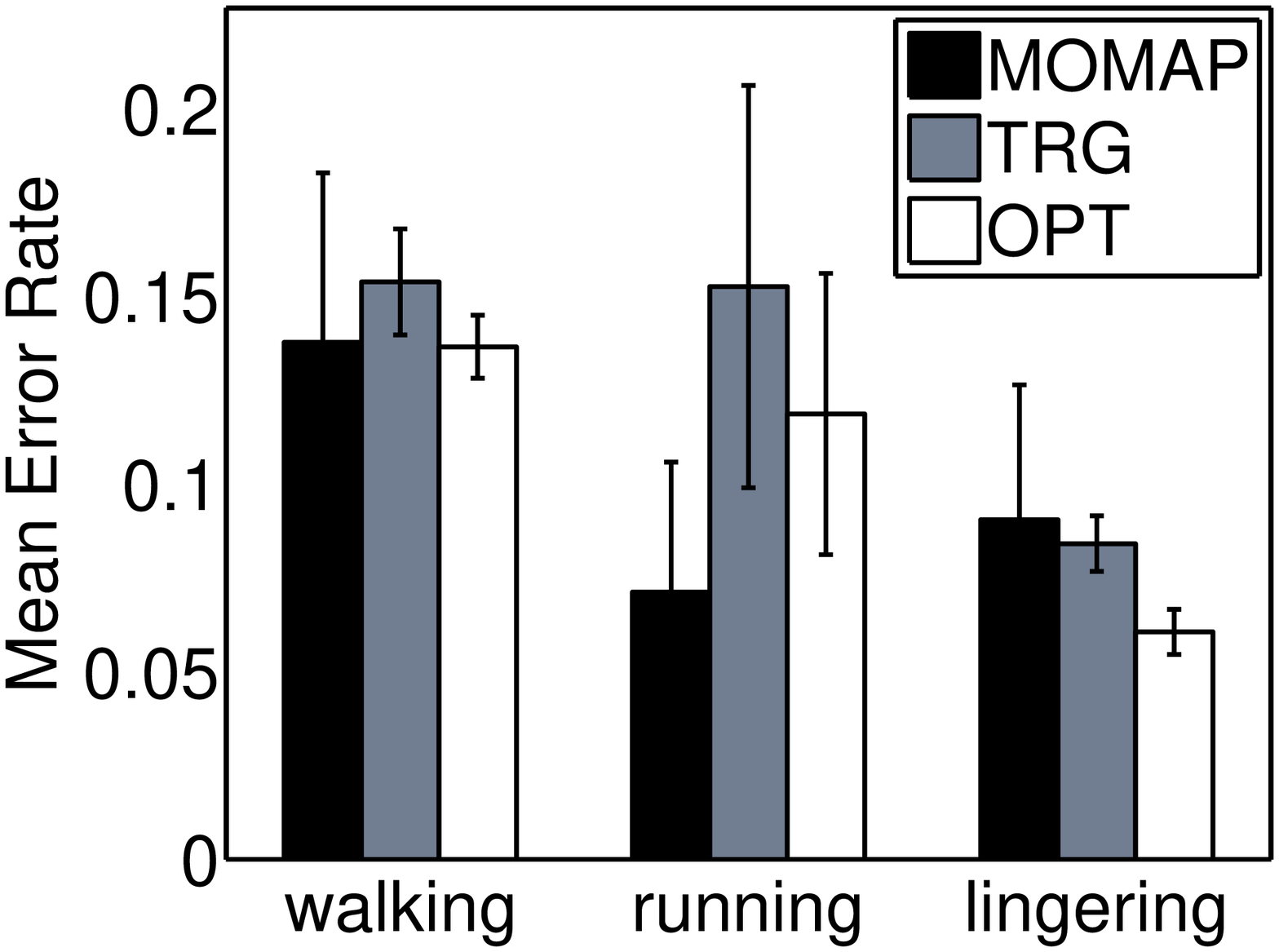}
   \label{subfig:12noisePerm}
 }
\subfigure[User 1 $\rightarrow$ User 4 ]{
   \includegraphics[trim = 12mm 0mm 0mm 0mm, clip,width=0.22\textwidth,height=0.21\textwidth] {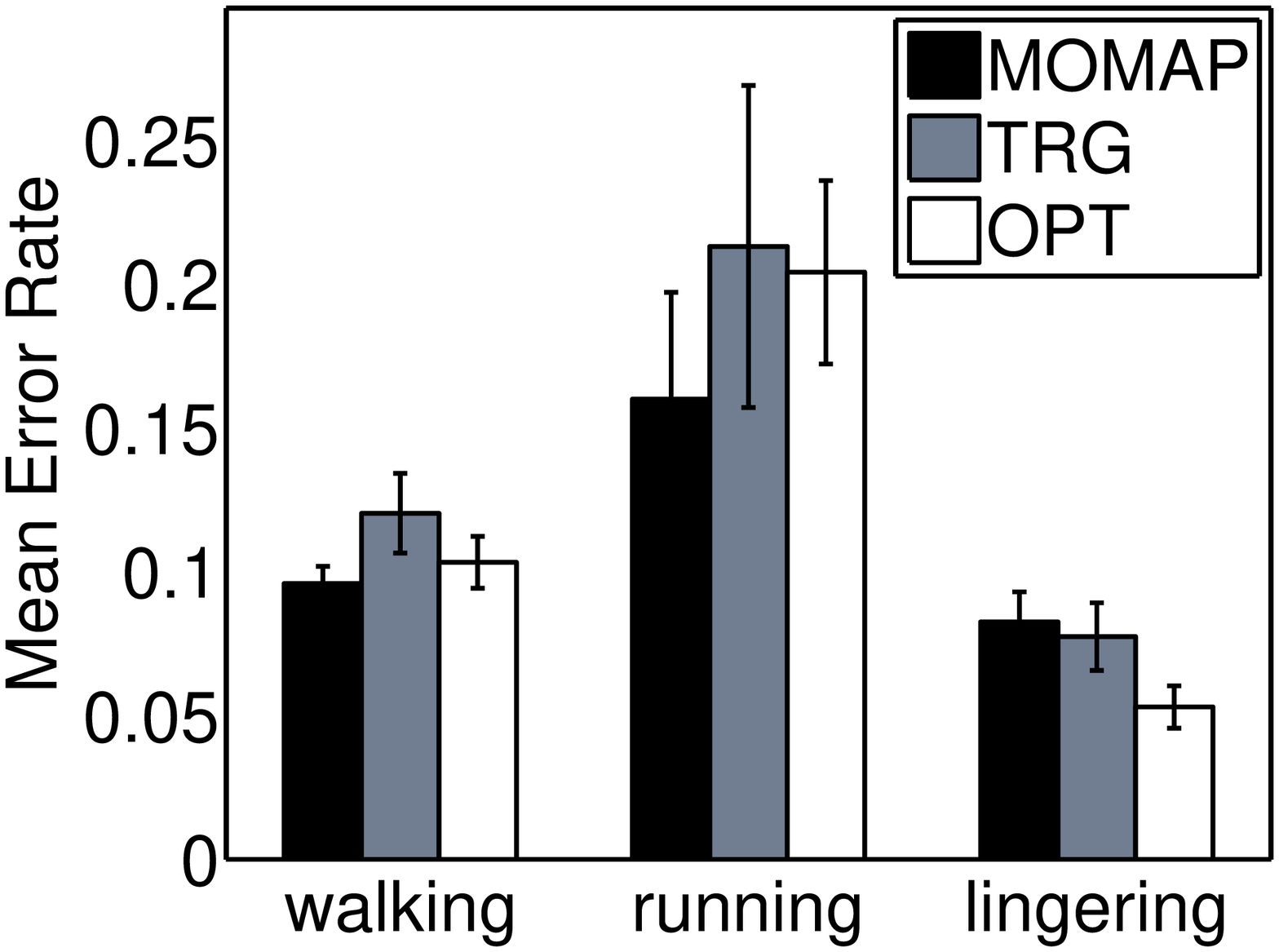}
   \label{subfig:41noisePerm}
 }
\subfigure[User 4 $\rightarrow$ User 5]{
   \includegraphics[trim = 12mm 0mm 0mm 0mm, clip,width=0.22\textwidth,height=0.21\textwidth] {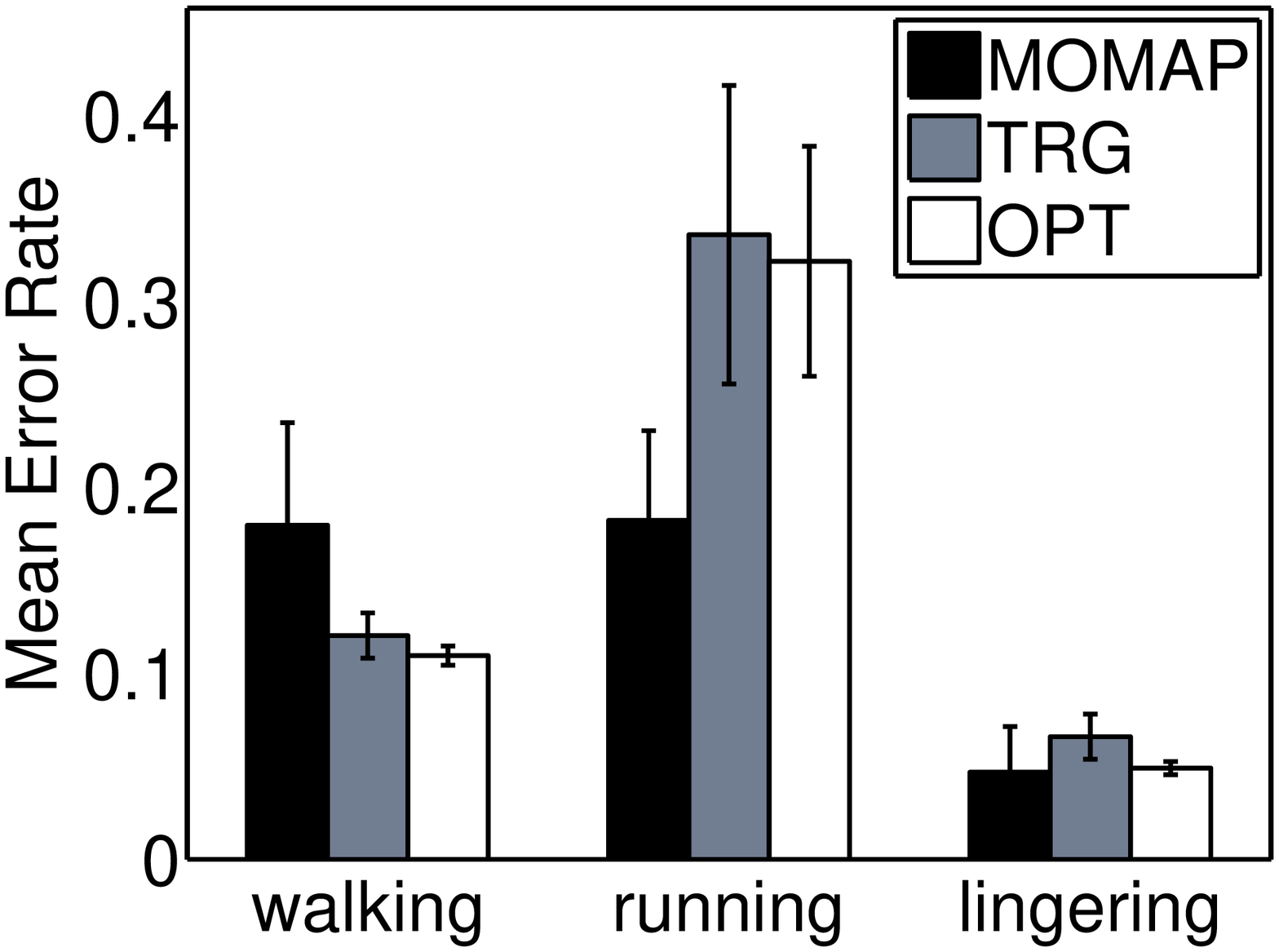}
   \label{subfig:54noisePerm}
 }
 \subfigure[User 5 $\rightarrow$ User 6]{
   \includegraphics[trim = 12mm 0mm 0mm 0mm, clip,width=0.22\textwidth,height=0.21\textwidth] {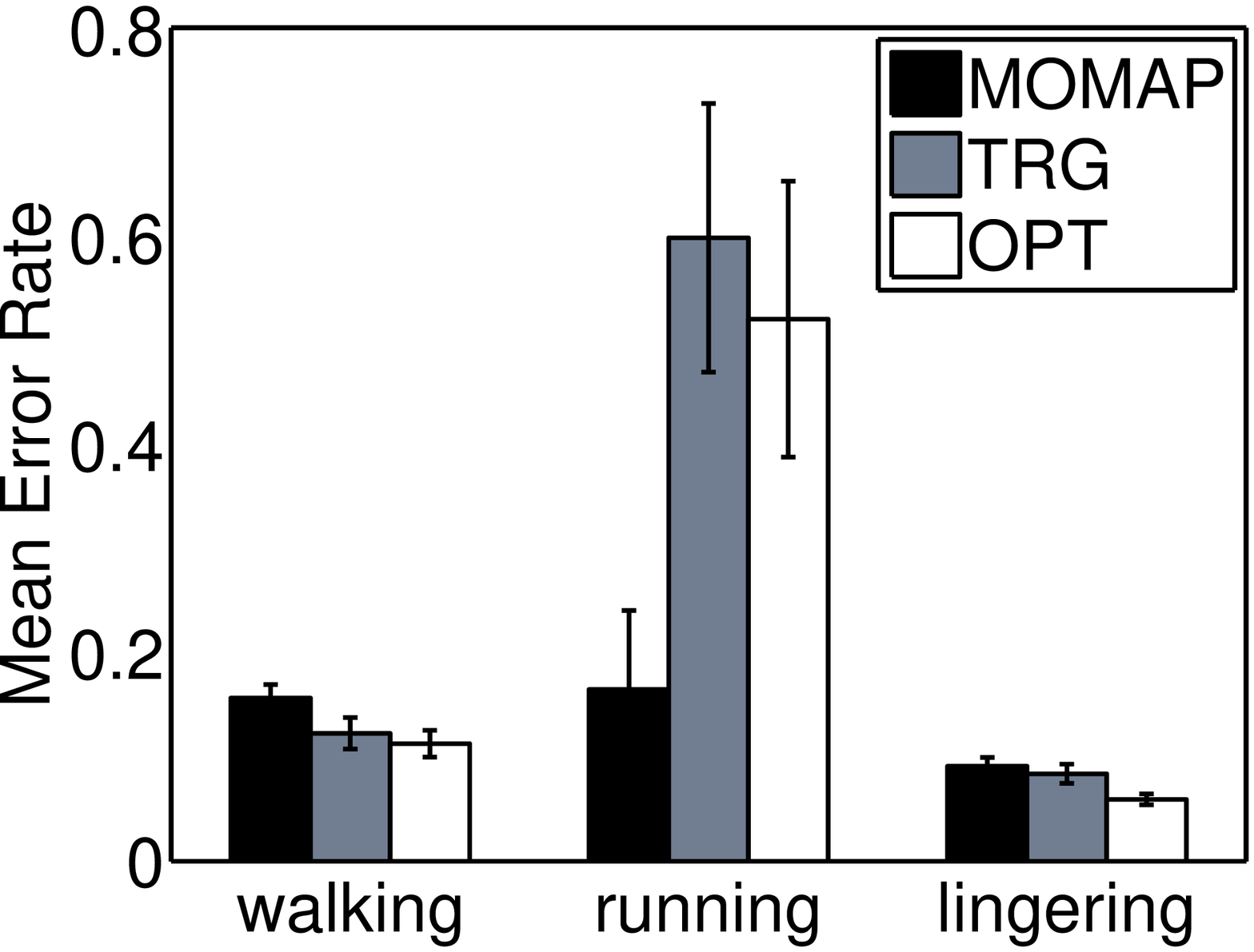}
   \label{subfig:65noisePerm}
 }
\caption{Multiple outlook setup  for two outlook with added noise features and randomly permuted features.\label{fig:noisePerm}}
\end{figure}

\begin{figure}[h!]
\centering
\subfigure[Users 1,4 and 6]{
   \includegraphics[width=0.3\textwidth] {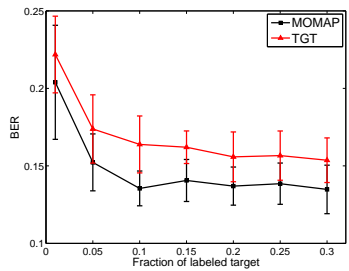}
   \label{subfig:146noisePerm}
 }
 \subfigure[Users 2,3 and 5]{
   \includegraphics[width=0.3\textwidth] {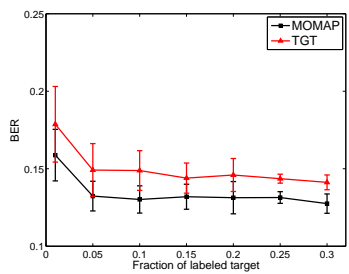}
   \label{subfig:146noisePerm}
 }

\caption{Multiple outlooks learning for mixture of $m=3$ outlooks. Noise features are added to each outlook and then the features are randomly permuted.\label{fig:MultnoisePerm}}
\end{figure}

\end{document}